\documentclass[letterpaper,twocolumn,10pt]{article}
\usepackage{usenix}

\usepackage{tikz}
\usepackage{amsmath}

\usepackage{filecontents}

\usepackage{amsthm}

\newtheorem{definition}{Definition}
\newtheorem{theorem}{Theorem}
\usepackage{dblfloatfix}
\usepackage{afterpage}
\usepackage{algorithm}
\usepackage{algpseudocode}
\usepackage{array}
\usepackage{multirow}
\usepackage{makecell}
\usepackage{tabularx}
\usepackage{booktabs}
\usepackage{amsmath}
\usepackage{cleveref}
\usepackage{amssymb}
\usepackage{subcaption}
\usepackage{comment}

\begin{document}

\date{}

\title{\Large \bf CAMA: Exploring Collusive Adversarial Attacks in c-MARL}


\author{
	\rm Men Niu$^{1,2}$ Xinxin Fan$^{1,2}$ Quanliang Jing$^{1}$ Shaoye Luo$^{1,2}$ Yunfeng Lu$^3$ 
	\\ $^1$Institute of Computing Technology, CAS $^2$UCAS, $^3$Beihang University, 
	\\ \textit{\{niumen24s, fanxinxin, jingquanliang, luoshaoye24s\}@ict.ac.cn}, \textit{lyf@buaa.edu.cn}
}

\maketitle

\begin{abstract}
Cooperative multi-agent reinforcement learning (c-MARL) has been widely deployed in real-world applications, such as social robots, embodied intelligence, UAV swarms, etc. Nevertheless, many adversarial attacks still exist to threaten various c-MARL systems. At present, the studies mainly focus on single-adversary perturbation attacks and white-box adversarial attacks that manipulate agents’ internal observations or actions. To address these limitations, we in this paper attempt to study collusive adversarial attacks through strategically organizing a set of malicious agents into three collusive attack modes: Collective Malicious Agents, Disguised Malicious Agents, and Spied Malicious Agents. Three novelties are involved: i) three collusive adversarial attacks are creatively proposed for the first time, and a unified framework CAMA for policy-level collusive attacks is designed; ii) the attack effectiveness is theoretically analyzed from the perspectives of disruptiveness, stealthiness, and attack cost; iii) the three collusive adversarial attacks are technically realized through agent’s observation information fusion, attack-trigger control. Finally, multi-facet experiments on four SMAC II maps are performed, and experimental results showcase the three collusive attacks have an additive adversarial synergy, strengthening attack outcome while maintaining high stealthiness and stability over long horizons. Our work fills the gap for collusive adversarial learning in c-MARL.
\end{abstract}

\section{Introduction}

Cooperative multi-agent reinforcement learning (c-MARL) has been widely applied to many scenarios ranging from unmanned aerial vehicle (UAV) control \cite{ekechi2025survey}, social simulation \cite{haupt2024formal}, to energy dispatch~\cite{wilk2024multi}, enabling agents to make decisions and coordinate based on their own local observations in a shared environment, while jointly learning policies to optimize a global objective. However, the existing studies~\cite{lin2020robustness} have shown that MARL system exhibits notable vulnerabilities when adversarial attacks exist, i.e., the malicious agents can easily induce the failure of agent cooperation or degradation of coordination services. 

Early adversarial attacks~\cite{lin2020robustness,liu2023efficient} primarily apply perturbations at the observation and action levels through directly modifying the victim agents’ observations or output actions to further mislead their policy decisions. These attacks typically require the attacker to fully acquire the MARL-system states, access to each agent’s internal model, and manipulation of observation/action channels. Therefore, this kind of white-box adversarial attacks cannot work in practical environments. Consequently, subsequent research gradually shifts from simple observation-level attacks to sophisticated policy-level attacks, in which a malicious agents  launch adversarial policy by design, then mislead good agents to learn irrational decision-making or disrupt cooperation~\cite{gleave2020adversarial, wu2021adversarial,guo2021adversarial,li2025attacking}.

Although policy-level attacks are more consistent with real-world applications from the black-box viewpoint, the research in this line is extremely limited to date. Most studies currently focus on single-adversary attacks~\cite{li2025attacking}, disregarding the in-depth investigation on collusive misbehavior among malicious agents. It is widely recognized that the malicious agents may bring stronger threats to c-MARL system once they organize together to strategically and collusively perform misleading policies. To fill the gap, we in this paper focus on more-sophisticated collusive adversarial attacks by multiple agents, instead of simple isolated adversarial attack by single agent. In collusive adversarial attacks, each malicious agent maintains its own policy and observation channel while acting collusively to pursue destructive objectives. This results in complex misbehavior patterns in the policy space, rather than simple local perturbations to observations or actions.

Nevertheless, achieving the sophisticated collusive adversarial attacks for c-MARL must encounter the following new challenges: i) \textbf{Adverseness-Augmented Collusion.} Though malicious agents can benefit from the collusive misbehavior, how to design information-sharing mechanisms and division-of-labor fashion from a deep learning perspective to surpass the effectiveness of independent attacks remains a hard problem; ii) \textbf{Balancing }\textbf{A}\textbf{ttack }\textbf{E}\textbf{ffectiveness and Stealthiness.} To persistently launch adversarial attacks without being detected, malicious agents also need to learn how to keep covert. If the attack signals are too weak, the attack cannot cause meaningful damage to the c-MARL system. Hence, how to accomplish an appropriate balance between attack effectiveness and stealthiness faces a significant challenge; iii) \textbf{Multi-Role Coordination.} To collusively achieve severe attacks in c-MARL environment, malicious agents ought to play different roles, for instance, some agents should focus on information gathering and infiltration, while others are responsible to trigger disruptive actions. Designing such collusive attack modes remains an open challenge. These gaps underscore the need for new theoretical and methodological frameworks to systematically study how collusive adversarial attacks arise in c-MARL systems and to characterize their attack effectiveness.

To systematically address the challenges above, in this paper, we propose a collusive adversarial multi-agent framework CAMA. Malicious agents are organized into groups and perform infiltration and disruption through strategically launching policy-level adversarial attacks. In a nutshell, three main contributions are involved:

\textbf{Collusive Adversarial Multi-Agent Framework.} We propose a collusive adversarial multi-agent framework to disclose the potential threats to c-MARL systems for the first time. By modeling three attack modes—collective, disguise, and spied collusion—CAMA fills a critical gap in the study of collusive adversarial attacks in multi-agent systems.

\textbf{\textbf{Formulation and Theoretical Verification. }}We formally model the three collusive adversarial attacks and develop a unified theoretical framework to capture inter-adversary coordination and the fundamental impact on c-MARL systems. Based on measures of attack effectiveness, we reveal the coordination mechanisms and adversarial characteristics of the three attack modes, which provides a verifiable theoretical foundation for studying collusive attacks for c-MARL.

\textbf{\textbf{Multi-Facet Empirical Validation. }}We conduct extensive experiments in the open-source platform SMAC II ~\cite{samvelyan2019starcraft} with four complex maps. Compared to non-collusive adversarial attacks, multi-facet experimental results show that our proposed CAMA not only substantially improves attack efficacy and influence scope, but also maintains high stealthiness and stability over multi-adversary collusive interactions.

\section{Problem Statement}

\subsection{Why Study Collusive Adversarial Attacks}

Recently, the adversarial perturbations in the domain of c-MARL have advanced beyond observation- and action-level attacks~\cite{liu2023efficient,lin2020robustness} to policy-level attacks~\cite{gleave2020adversarial, wu2021adversarial,guo2021adversarial,li2025attacking}. However, the current studies focus on a single malicious agent, disregarding the collaboration in the c-MARL environment. It is well known that the fundamental property of c-MARL lies in that multiple agents collaboratively perform various tasks, and the performance depends on group-level coordination among agents. However, conversely, such a cooperation may introduce more severe structural risks, i.e., collusive adversarial attacks. Compared to single-attacker local perturbations, the collusive multi-agent adversarial behavior yielding global perturbations should incur stronger threat to c-MARL systems. 

From the viewpoint of practical applications, we think such a collusive adversarial attack can occur in many collaborative systems/networks, such as social-bot networks~\cite{le2022socialbots}, multi-player/agent game~\cite{samvelyan2019starcraft}, financial and economic
game-theoretic systems~\cite{agrawal2025evaluating}, etc. These systems can be abstracted as c-MARL systems. However, to the best of our knowledge, the existing work on c-MARL still lacks systematic studies on collusive adversarial attacks. Therefore, we aim to bridge this gap by introducing collusive adversarial attacks via organized coordination among multiple malicious agents.

\subsection{How Existing Works Investigate Adversarial Policy}

The existing studies on policy-level adversarial attacks in c-MARL can be summarized into two categories: i) \textbf{Outside Adversarial Policy Attacks.} The attacker learns adversarial policies as an external agent and interacts with victim agents to induce systemic failures. Gleave et al.~\cite{gleave2020adversarial} report that even a strong policy can be systematically influenced by an adversarial policy through regular interactions. Building on this, Wu et al.~\cite{wu2021adversarial} use saliency analysis to identify sensitive observation features and apply targeted attacks. Guo et al.~\cite{guo2021adversarial} also extend such adversarial-policy learning into general zero-sum games. Subsequently, adversarial policies are also investigated in partially observable multi-agent systems; ii) \textbf{Inside Malicious Agent Injection Attack.} The attacker (malicious agent) joins the c-MARL team as a member and interacts with others. It induces harmful behaviors that degrade team performance. For instance, Li et al.~\cite{li2025attacking} introduce a policy-level single-agent attack to disrupt the system without controlling the environment or tampering with good agents’ observation channels. Such black-box adversarial attacks are more practical in real-world scenarios. However, existing studies still focus on single-agent attacks, without modeling collusive adversarial attacks among multiple agents. 

\subsection{Attack Effectiveness and Stealthiness Should Be Co-considered}

To date, most policy-level adversarial attacks model the victim-attacker interaction as a partially observable Markov game~\cite{littman1994markov} or Dec-POMDP~\cite{bernstein2002complexity}, in which the attack objective is typically defined as minimizing the victim team’s return or win rate, without considering attack stealthiness and long-term sustainability. Recently, some works attempt to incorporate stealth by triggering attacks sparsely or limiting the control strength to reduce abnormal signals, which can significantly degrade attack effectiveness~\cite{yang2025optimized,hu2022sparse}. Therefore, for malicious agents, a more cunning manner is to conceal the attack properties through disguise and spied behaviors. In this way, malicious agents can achieve strong attack impact while maintaining stealth and avoiding detection. 

\section{CAMA: Collusive Multi-Agent Framework}
In this section, we formally define three types of collusive adversarial attacks, and develop principled measures on attack effect for the three types of attacks regarding observation sharing, temporal control, and agent-role differentiation. 

\subsection{Preliminaries}
\subsubsection{Environment Modeling}

Typically, the multi-agent system can be formalized as a finite-horizon Decentralized Partially Observable Markov Decision Process (finite-horizon Dec-POMDP)~\cite{bernstein2002complexity}, denoted by
\begin{equation*}
\mathcal{G} = (S, \{A_i\}_{i \in \mathcal{NUM}}, T, \{O_i\}_{i \in \mathcal{NUM}}, R, \gamma)
\end{equation*}

The separate components are defined as follows:
\begin{itemize}
  \item \textbf{State Space.}
  We use $S$ to denote the set of global environment states, and let
  $s_t \in S$ denote the global state at time step $t$.
  The initial state is drawn from an initial distribution
  $\rho_0(s_0)$.

  \item \textbf{Action Spaces.}
  For each agent $i$, let $\mathcal{A}_i$ denote its action space.
  We partition agents into a set of good agents $\mathcal{N}$, which
  cooperate to accomplish the team task, and a set of malicious agents
  $\mathcal{M}$, which are adversarial agents aiming to disrupt the task.
  At time step $t$, agent $i$'s action is $a_{i,t} \in \mathcal{A}_i$.
  The joint action at time $t$ is
 $\mathbf{a}_t = (a_{i,t})_{i \in \mathcal{NUM}}$.

  \item \textbf{Transition Dynamics.}
  $T : S \times \mathcal{A} \times S \rightarrow [0,1]$ denotes the
  transition function, where
  $T(s_{t+1} \mid s_t, \mathbf{a}_t)$
 gives the probability of transitioning from
  $s_t$ to $s_{t+1}$ under joint action $\mathbf{a}_t$.

  \item \textbf{Observation Spaces.}
  For each agent $i$, $\mathcal{O}_i$ is its observation space.
  At time step $t$, agent $i$ receives a local observation
  $o_{i,t} \in \mathcal{O}_i$.
  Each agent selects actions and updates its policy based on its own
  observation $o_{i,t}$.

  \item \textbf{Team Reward.}
  $R : S \times \mathcal{A} \rightarrow \mathbb{R}$ is the team reward
  function.
  The instantaneous reward at time $t$ is
  $r_t = R(s_t, \mathbf{a}_t)$.

  \item \textbf{Discount Factor.}
  $\gamma \in [0,1]$ is the discount factor that trades off short-term and
  long-term returns.
  A small $\gamma$ emphasizes immediate rewards, while a value of
  $\gamma$ close to $1$ emphasizes future returns.
\end{itemize}

\subsubsection{Policy Settings}
We partition agents into two parts: good agents $\mathcal{N}$ and malicious agents $\mathcal{M}$. Four policies are accordingly defined:
\begin{itemize}
  \item \textbf{Good-Agent Policy.}
  For each good agent $i \in \mathcal{N}$, the normal policy is
  $\pi_i^{N}(a_{i,t} \mid o_{i,t})$, which specifies the probability that good
  agent $i$ selects action $a_{i,t}$ given its current observation $o_{i,t}$.
  The joint normal policy is
  $\Pi^{N} = \{\pi_i^{N}\}_{i \in \mathcal{N}}$.

  \item \textbf{Malicious-Agent Policy.}
  For each malicious agent $i \in \mathcal{M}$, the adversarial policy is
  $\pi_i^{M}$.
  When a malicious agent mimics normal behavior, we set
  $\pi_i^{M} = \pi_i^{N}$.
  The joint adversarial policy is
  $\Pi^{M} = \{\pi_i^{M}\}_{i \in \mathcal{M}}$.

  \item \textbf{System Joint Policy under Non-Adversarial Behaviors (SJPNAB).}
  All agents follow normal policies.
  The system joint policy is
  $\Pi^{\text{normal}} = \{\pi_i^{N}\}_{i \in \mathcal{N} \cup \mathcal{M}}$.

  \item \textbf{System Joint Policy under Adversarial Behaviors (SJPAB).}
  When malicious agents execute adversarial policies, the system joint policy
  becomes
  $\Pi^{\text{adv}} = \{\pi_i^{N}\}_{i \in \mathcal{N}} \cup
  \{\pi_i^{M}\}_{i \in \mathcal{M}}$.
  Given a joint policy $\Pi$, the expected discounted return is
$
J(\Pi)
= \mathbb{E}_{\Pi}
  \left[
    \sum_{t=0}^{T} \gamma^t R(s_t, \mathbf{a}_t)
  \right].
$

\end{itemize}

\subsection{Adversarial Efficacy}
To evaluate collusive adversarial attacks, CAMA introduces an \emph{Adversarial Efficacy Function (AEF)} under a unified criterion with respect to different adversarial policies:
\begin{equation}
J(\Pi^{\text{adv}}) = \alpha D + \beta S - \gamma C,
\label{eq:aef}
\end{equation}
where $D$, $S$, and $C$ denote the \emph{disruptiveness}, \emph{stealthiness}, and \emph{attack cost}, respectively, the coefficients $\alpha$, $\beta$, and $\gamma$ are associated weights. We detail the three properties as follows:

\textbf{Disruptiveness.}
This term measures the performance degradation of good agents.
We define it as
\begin{equation}
D = J^{\text{normal}} - J^{\text{adv}},
\label{eq:D_def}
\end{equation}
where
$J^{\text{adv}} = \mathbb{E}_{\Pi^{\text{adv}}}
\!\left[\sum_{t=0}^{T} \gamma^{t} R(s_t, \mathbf{a}_t)\right]$
is the expected discounted return under the SJPAB $\Pi^{\text{adv}}$,
and
$J^{\text{normal}} = \mathbb{E}_{\Pi^{\text{normal}}}
\!\left[\sum_{t=0}^{T} \gamma^{t} R(s_t, \mathbf{a}_t)\right]$
is the expected discounted return under the SJPNAB $\Pi^{\text{normal}}$.
A larger $D$ indicates more severe performance degradation.
To facilitate measurement and optimization, we further express $D$ in an additive form over time as a first-order approximation.
Let $x_t \in \{0,1\}$ indicate whether an adversarial attack is triggered at time $t$, and let $g_t$ denote the expected return reduction caused by the malicious agents' attack at time $t$.
Then, we have
$D(x) = \sum_t g_t x_t$.

\textbf{Stealthiness.}
In multi-agent systems, the detector’s core goal is to distinguish behavioral
trajectories generated by good agents and malicious agents~\cite{kazari2023decentralized}.
The closer the two trajectory distributions are, the harder it is for the
detector to distinguish them, thus leading to higher stealthiness for
malicious agents.
We therefore define stealthiness from the perspective of distinguishability.
For each malicious agent $i \in \mathcal{M}$, let
$\tau_i = (a_{i,0}, a_{i,1}, \ldots, a_{i,T})$ denote its action trajectory.
Let $P_i^{M}$ be the trajectory distribution induced by the adversarial policy
$\pi_i^{M}$, and $P_i^{N}$ be the trajectory distribution induced by the normal
policy $\pi_i^{N}$.
Thus, the KL divergence between them can be defined as
\begin{equation}
D_{\mathrm{KL}}\!\left(P_i^{M} \,\|\, P_i^{N}\right)
= \mathbb{E}_{\tau_i \sim P_i^{M}}
\left[ \log \frac{P_i^{M}(\tau_i)}{P_i^{N}(\tau_i)} \right].
\label{eq:kl_traj}
\end{equation}

We define stealthiness as the average inverse divergence across malicious
agents:
\begin{equation}
S = \frac{1}{|\mathcal{M}|}
\sum_{i \in \mathcal{M}}
\frac{1}{D_{\mathrm{KL}}\!\left(P_i^{M} \,\|\, P_i^{N}\right) + \varepsilon},
\label{eq:stealthiness}
\end{equation}
where $\varepsilon > 0$ is a very small constant to avoid a zero denominator.
When a malicious agent behaves almost identically to the normal policy,
i.e., $D_{\mathrm{KL}} \rightarrow 0$, then $S$ becomes large, indicating
extremely high stealthiness; conversely, larger behavioral deviation yields
larger $D_{\mathrm{KL}}$, which implies smaller $S$, leading the attack to be
detected easily.

To facilitate fine-grained optimization and analytics, we further decompose the
KL divergence at the time-step level.
If the behavioral trajectory likelihood factorizes over time (under the
policy-induced rollout), namely
$P_i^{M}(\tau_i) = \prod_{t=0}^{T} \pi_i^{M}$ and
$P_i^{N}(\tau_i) = \prod_{t=0}^{T} \pi_i^{N}$, then alternatively,
\begin{equation}
D_{\mathrm{KL}}\!\left(P_i^{M} \,\|\, P_i^{N}\right)
= \mathbb{E}_{\tau_i \sim P_i^{M}}
\left[ \sum_{t} \log \frac{\pi_i^{M}}{\pi_i^{N}} \right].
\label{eq:kl_step}
\end{equation}
Accordingly, we define the per-step exposure as
$\delta_{i,t} = \log \frac{\pi_i^{M}}{\pi_i^{N}}$,
and introduce an indicator $x_{i,t} \in \{0,1\}$ to denote whether malicious
agent $i$ executes an adversarial attack at time $t$.
The cumulative exposure is
$\Delta = \sum_{i \in \mathcal{M}} \sum_t x_{i,t} \, \delta_{i,t}$.
This construction satisfies
$\mathbb{E}[\Delta] \propto \sum_{i \in \mathcal{M}}
D_{\mathrm{KL}}(P_i^{M} \,\|\, P_i^{N})$,
that is to say, the expected cumulative exposure is proportional to the KL
discrepancy between adversarial and normal behaviors.
Hence, in the subsequent analysis we use a stealthiness surrogate $S(\Delta)$,
where $S(\cdot)$ is a monotone non-increasing function of $\Delta$.
This assumption reflects that larger deviations from normal behaviors are more
easily detected, thus reducing stealthiness.
To enable tractable theoretical analysis and attack optimization, we assume the
stealthiness function satisfies a Lipschitz-type lower-bound condition.
Specifically, there exists a constant $k > 0$, such that for any
$\Delta_1 \ge \Delta_2$,
\begin{equation}
S(\Delta_2) - S(\Delta_1) \ge k(\Delta_1 - \Delta_2).
\label{eq:lipschitz_lower}
\end{equation}
This condition characterizes the minimum rate at which stealthiness decreases
with increasing exposure.

\textbf{Attack Cost.}
For malicious agents, collusively executing adversarial policy would incur
more costs, such as additional communication bandwidth, energy consumption,
and computational overhead.
To capture this aggregated cost, we define a non-negative per-step cost
$c_t \ge 0$ at each time step $t$,
\begin{equation}
C = \sum_{t=0} c_t x_t.
\label{eq:attack_cost}
\end{equation}
\subsection{Collusive Adversarial Attacks}
In this paper, we assume multiple malicious agents can form diverse modes of collusion through observation-information sharing, attack-trigger controlling, and agent-role differentiation. We creatively define three adversarial attacks: Collective Malicious Agents (L1), Disguised Malicious Agents (L2), and Spied Malicious Agents (L3).

\subsubsection{Collective Malicious Agents (L1)}
In c-MARL, an agent can only condition its decision on the local observation
information, thus a malicious agent $i$ makes decision at time step $t$ via its
local observation $o_{i,t}$, i.e., following
$\pi_i^{M}(a_{i,t} \mid o_{i,t})$, this limited information inevitably constrains
its disruptive capability.
To address this problem, an intuitive solution is to extend into a collusive
scenario where multiple malicious agents form a collective to share
observations, and each agent makes decisions based on the shared information.
Thus, we give the following definition.

\begin{definition}[\textbf{Collective Malicious Agents}]
At each time step $t$, all malicious agents share their observations and
construct an integrated representation using a fusion function
$F(\cdot): \tilde{o}_t = F(\{o_{i,t}\}_{i \in \mathcal{M}})$.
Then, each malicious agent makes decisions based on its local observation and
the aggregated representation.
The joint policy
$\Pi^{\text{L1}} = \{\pi_i^{M}(a_{i,t} \mid \tilde{o}_t)\}_{i \in \mathcal{M}}$
is employed to provide a more complete global context in partially observable
environments.
This coordinated decision-making can further induce nonlinear amplification
of disruptive effects in multi-agent interactions.
\end{definition}

\begin{theorem}
In a finite-horizon Dec-POMDP, if malicious agents can share and fuse
observations, then the disruptiveness under CMA attack satisfies
$D_{\mathrm{L1}} \ge D_{\mathrm{ind}}$, where $D_{\text{L1}}$ denotes the disruptiveness under CMA, and
$D_{\text{ind}}$ denotes the disruptiveness under multiple malicious agents
acting independently without collusion.
\end{theorem}

\begin{proof}
See Appendix~\ref{app:A}.
\end{proof}

\subsubsection{\textbf{Disguised Malicious Agents (L2)}}
As discussed above, CMA perturbs good agents at each time step. However, malicious agents cannot hold long-term disruptiveness resulting from significant deviation on behavioral trajectories from normal patterns generated by good agents, making it easier to identify them. To tackle this predicament, building on CMA, we also propose a sophisticated attack, i.e. Disguised Malicious Agents (DMA).
\begin{definition}[\textbf{Disguised Malicious Agents}]
Malicious agents coordinate via observation sharing for joint decision-making.
They introduce selective temporal control over whether to launch an adversarial
attack.
Specifically, attacks are triggered only at time steps deemed to be
high-valuable, while at all other time steps the malicious agents strictly
imitate normal behavior, so that the majority of decisions match the normal
policy.
This yields a low-frequency yet high-impact stealthy long-horizon attack.
\end{definition}
To quantify the per-step trade-off between attack effectiveness, stealthiness degradation, and resource overhead, we introduce the \emph{adversarial advantage}
(AA) $\Phi_t$, which captures the relative contributions of the immediate
disruption gain, the stealthiness loss, and the resource overhead when
triggering an attack at time step~$t$.
If $\Phi_t < 0$, the benefit of attacking at time~$t$ is insufficient to offset
the stealthiness degradation and resource consumption, and we call such a time
step \emph{ineffective}.
Conversely, if $\Phi_t > 0$, time step~$t$ is considered \emph{effective}.
Based on this notion, we can formally compare DMA and CMA under the Adversarial Efficacy Function (AEF) with the following theorem.
\begin{theorem}
If DMA triggers attacks only at effective time steps with $\Phi_t > 0$, then its
adversarial efficacy is no worse than that of CMA, i.e.,
$J_{\text{L2}} \ge J_{\text{L1}}$.
Moreover, if at least one ineffective time step with $\Phi_t \le 0$ is excluded,
the strict inequality holds:
$J_{\text{L2}} > J_{\text{L1}}$.
\end{theorem}
\begin{proof}
See Appendix~\ref{app:B}.
\end{proof}

\begin{figure*}[t]
  \centering
  \includegraphics[width=1\textwidth]{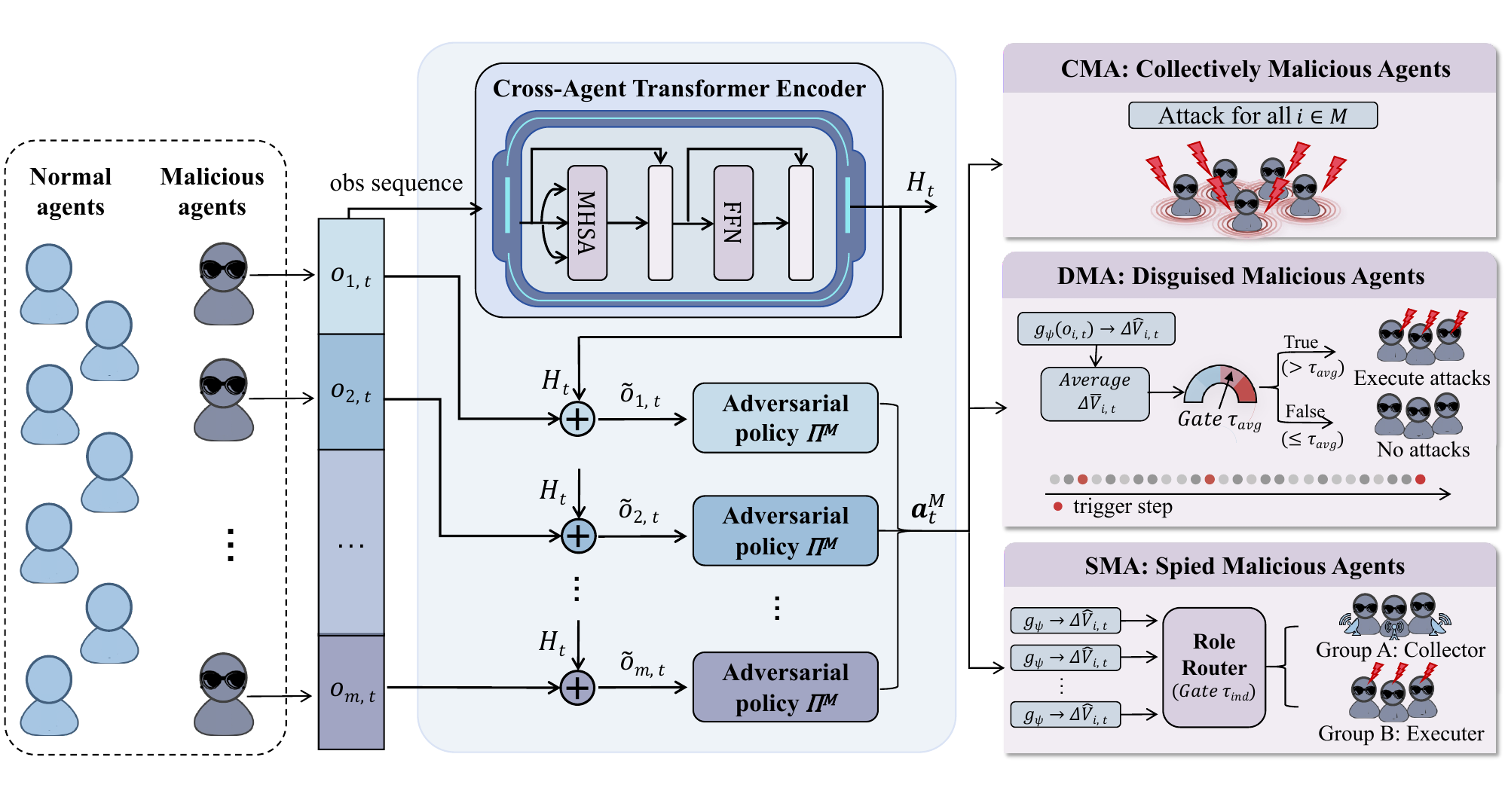}
  \caption{The framework provides a unified characterization of three collusive adversarial attacks: CMA, DMA, and SMA.
On the left, local observation information from normal agents and malicious agents are collected, where the observation sequences of all malicious agents are first fused by a cross-agent Transformer encoder to produce a shared context-aware representation $\mathbf{H}_t$.
Subsequently, each malicious agent augments its local observation with the $\mathbf{H}_t$ and feeds the fused features into the adversarial policy network.
On the right, the execution mechanisms of the three attack modes are illustrated: in CMA, all malicious agents collectively launch attacks at all time steps;
in DMA, attacks are triggered only at high-value time steps through a value-driven attack triggering mechanism;
in SMA, a role assignment mechanism is further introduced, where a subset of malicious agents act as collectors, while the remaining malicious agents execute adversarial attacks.
}
  \label{fig:cama}
\end{figure*}

\subsubsection{Spied Malicious Agents (L3)}
In DMA, the malicious agents refer to a binary trigger
$x_t \in \{0,1\}$ as the indicator to launch adversarial perturbation whenever
the time step is deemed valuable.
To further improve stealthiness and resource efficiency, building on DMA,
we propose a more sophisticated adversarial attack: Spied Malicious Agents (SMA).

\begin{definition}[\textbf{Spied Malicious Agents}]
Malicious agents are divided into two groups.
Malicious agents (a.k.a. Collectors) in Group~A, keeping their behaviors fully
aligned with good agents, do not perform adversarial attacks with the purpose of
maintaining the maximum concealment.
However, they will collect local observation information and transmit to
Group~B.
In contrast, malicious agents (a.k.a. Executors) in Group~B combine their local
observation with shared information from Group~A to integrate global
observation information, then perform adversarial perturbations to the c-MARL system.

\end{definition}
Formally, at time step $t$, for each malicious agent $i \in \mathcal{M}$, we introduce an adversarial trigger $x_{i,t} \in \{0,1\}$.
If $x_{i,t} = 1$, agent $i$ acts as an executor and participates in the perturbations at time $t$; if $x_{i,t} = 0$, agent $i$ acts as a
collector, executes the normal policy, and only performs information sharing.

To quantify the adversarial advantage of a single malicious agent, we fix the
behaviors of all other malicious agents to follow the DMA and define the
following terms:

\begin{itemize}
  \item $g_{i,t}$: the change in disruptiveness $D$ when switching $x_{i,t}$ from $0$ to $1$ (with other malicious agents fixed to L2).

  \item $\delta_{i,t}$: the change in the cumulative exposure $\Delta$ under the same condition.

  \item $c_{i,t}$: the change in the total attack cost $C$ under the same condition.
\end{itemize}

Accordingly, we define the adversarial advantage for each agent-time
pair $(i,t)$ as
\begin{equation}
\Phi_{i,t}
= \alpha g_{i,t} - \beta k \delta_{i,t} - \gamma c_{i,t}.
\label{eq:advantage_phi}
\end{equation}
Here $\alpha, \beta, \gamma > 0$ are the weights of disruptiveness $D$,
stealthiness $S$, and attack cost $C$ in AEF, $k > 0$ is the
lower-bound coefficient on the minimum decrease rate of stealthiness. Based on the definition above, we deduce the following theorem.
\begin{theorem}
Suppose the SMA assigns each malicious agent at each time step according to
\begin{equation*}
\begin{cases}
\Phi_{i,t} \le 0 \;\Rightarrow\; x_{i,t} = 0 \quad (\text{assigned to Group A}), \\
\Phi_{i,t} > 0 \;\Rightarrow\; x_{i,t} = 1 \quad (\text{assigned to Group B}).
\end{cases}
\end{equation*}

Then inequality $\mathcal{J}_{L3} \ge \mathcal{J}_{L2}$ holds. Moreover, if there exists at least one pair $(i,t)$ such that $\Phi_{i,t} < 0$ and it is assigned to Group~A, the inequality is strict, i.e., $\mathcal{J}_{L3} > \mathcal{J}_{L2}$.
\end{theorem}

\begin{proof}
See Appendix~\ref{app:C}.
\end{proof}

\textbf{Summary.}
Overall, CMA (L1), DMA (L2), and SMA (L3) form a three-level progressive adversarial hierarchy of our proposed collusive adversarial attacks in c-MARL system. Their evolution follows a clear clue: from always-on collective adversarial perturbations, to selective temporal concealment, and further to role-specialized coordination.

Under L1, all malicious agents synchronously launch disruptive perturbations, maximizing disruptive capability in a direct manner, which makes it easier to be detected. Building on L1, L2 introduces temporal selectivity via the adversarial advantage $\Phi_t$, triggering attacks only at effective time steps. This design explicitly balances disruptiveness and stealthiness, mitigating the excessive exposure inherent in L1. L3 further refines the decision granularity to the agent level by dynamically partitioning malicious agents into Group A and Group B. Group A is in charge of collecting and relaying observation information, enabling Group B to launch perturbations precisely at the most valuable moments, thereby realizing explicit specialization and coordination among malicious agents. This level preserves the stealthiness of L2 while achieving a higher-order form of collusion through role differentiation.

Theoretically, we can infer the following inequality:
$\mathcal{J}_{L1} \le \mathcal{J}_{L2} \le \mathcal{J}_{L3}$. 
To verify the correctness of the theoretical results, we next detail the techniques used to realize the three collusive adversarial attacks in the following sections.

\section{Implementation}

Inheriting the theoretical framework for the three collusive adversarial attacks, we develop practical training and execution pipelines to detail the implementations in c-MARL. An overview of the CAMA framework is illustrated in Fig.~\ref{fig:cama}.

\subsection{CMA: Transformer-based Observation-Information Fusion}
In partially-observable MARL environments, as is well known, each agent can only perceive local information via its observation $o_{i,t}$. Thus, for a single agent, it is difficult to obtain global policy context. Our designed CMA aims to address this problem by integrating multiple agents' observation-information into a shared adversarial context, enabling joint adversarial perturbations on the c-MARL system.

A straightforward approach is to concatenate all malicious agents' observation vectors and integrate them into a single policy network. Nevertheless, such a design neither explicitly represents relational structures (e.g., spatial layout, role heterogeneity) among malicious agents, nor remains robust and generalizable when the number of malicious agents varies. To handle this, we design a shared representation-aware paradigm by inserting a cross-agent Transformer-based observation fusion module.
As shown in Fig.~\ref{fig:cama}, at time step $t$, let $\mathcal{M} = \{1, \ldots, m\}$ denote
the set of $m$ malicious agents with observations
$o_{1,t}, \ldots, o_{m,t}$. The observations of all malicious agents are organized as a sequence and processed by a Transformer encoder. Since multi-head self-attention builds learnable information channels across all malicious agents, a set of context-aware shared representations can be yielded:
$\
\mathbf{H}_t = \mathrm{TransformerEncoder}(o_{1,t}, \ldots, o_{m,t}).
$
For each malicious agent $i$, we use its local observation as the primary input
and treat the Transformer-aggregated shared representation $\mathbf{H}_t$ as a
global contextual conditioning signal to augment the agent's feature
representation, in this way, a fused (local $o_{i,t}$ \& global $\mathbf{H}_t$) input $\tilde{o}_{i,t}$ can be constructed. As a result, each malicious agent can make a decision based on fused observation $\tilde{o}_{i,t}$, enabling collusive behavior.

Our designed shared representation-aware paradigm has two advantages: i) self-attention can capture fine-grained dependencies among agents’ observations. It can encode collusive interaction structure in the latent space. ii) the sequence length (feature dimension) naturally scales with the number of malicious agents, facilitating cross-task transfer and varied experiments with different configurations.

\subsection{DMA: Temporal Gating-Triggered Attack}

As previously analyzed, CMA is highly susceptible to behavioral trajectory-based detectors, to improve the stealthiness, we implement DMA through introducing a temporal gating mechanism driven by PPO-style value estimator, i.e., the attacks are triggered only at high-value time steps, while at other time steps the malicious agents mimic normal actions. 

We employ the centralized critic in MAPPO~\cite{yu2022surprising}, that is,
when agents are trained with MAPPO, the critic maintains a state-value function
$V(s_t)$, used to estimate the expected discounted return from the global state
$s_t$. We then define the value difference between consecutive states as
$\Delta V_{t-1} = V_{\text{adv}}(s_t) - V_{\text{adv}}(s_{t-1})$, where
$V_{\text{adv}}(s_t)$ denotes the critic value under the adversarial policy,
serving as a proxy for the value increment from $s_{t-1}$ to $s_t$. Taking into
consideration this signal, we introduce a learnable stealthiness-focused module
$g_\psi$. Specifically, we take as input the concatenation of the
aforementioned fused feature $\tilde{o}_{i,t}$ and agent-specific features (e.g., agent ID and last action), and predict each malicious agent's value-increment estimate
$\Delta \hat{V}_{i,t} = g_\psi(\tilde{o}_{i,t})$.

During training, $\Delta V$ is used as a supervision indicator, and we fit
$g_\psi$ by minimizing the mean-squared error:
$\mathcal{L}_{\text{loss}} = \mathbb{E}\big[(g_\psi(\tilde{o}) - \Delta V)^2\big]$.
During inference, for each time step $t$, we compute
$\Delta \hat{V}_{1,t}, \ldots, \Delta \hat{V}_{m,t}$ for all malicious agents, and aggregate them by averaging to obtain a group-level score
$\Delta \bar{V}_t$. Subsequently, we compare it with a pre-defined threshold
$\tau_{\text{avg}}$. If $\Delta \bar{V}_t > \tau_{\text{avg}}$, the time step is triggered to execute adversarial attacks; otherwise, no attacks occur.

\subsection{SMA: Group-based Spied Attack}
Although DMA substantially improves stealthiness by temporally selective
adversarial attacks, it still synchronizes the adversarial actions of all
malicious agents. Considering more realistic scenarios, we allow a subset of
malicious agents to follow the normal policy and act as information collectors
(spies). These malicious agents contribute their local observations to another
group of malicious agents, forming the SMA. SMA follows the same model
architecture and training procedure as DMA.

We compute an adversarial value prediction for each malicious agent $i$, but
apply the threshold independently for each agent. If
$\Delta \hat{V}_{i,t} < \tau_{\text{ind}}$, agent $i$ is assigned to Group~A
(collector) at time step $t$ and does not trigger adversarial attacks; if
$\Delta \hat{V}_{i,t} > \tau_{\text{ind}}$, agent $i$ is assigned to Group~B
(executor) and triggers adversarial attacks.

Methodologically, DMA determines whether to launch adversarial attacks based on
the averaged individual predictions. In contrast, SMA determines which agents
should attack by maintaining individual prediction mechanisms, while the
remaining agents act as collectors. This simple change in gating granularity
enables the transition from disguise to spy without introducing additional
model complexity.

\subsection{Training and Execution Pipeline}
To facilitate systematic analysis and fair comparison, we design two complementary yet decoupled procedures. The first is a unified training pipeline, which collects behavioral trajectories under a specified attack mode (CMA/DMA/SMA) and jointly optimizes the adversarial policy network, value network, and mode-specific modules.
The second is an execution pipeline, which determines which malicious agents perform adversarial perturbations and how the perturbations are executed under different attack modes. More importantly, the three collusive adversarial attacks are encapsulated in this execution procedure in a decoupled form.

\textbf{Overall Training Procedure.}
Algorithm~\ref{alg:training_pipeline} summarizes the unified training pipeline.
CMA is trained jointly with $\pi^{M}$ via PPO, while DMA and SMA are trained using
supervised regression on $\Delta V$.

\begin{algorithm}[tbhp]
\caption{\textbf{Unified Training Pipeline}}
\label{alg:training_pipeline}
\begin{algorithmic}[1]
\Require Fixed normal-policy network $\pi^{N}$ for good agents; adversarial-policy network $\pi^{M}$ equipped with CMA, DMA, and SMA modules; value network
$V(s)$; maximum number of training iterations $K$; number of sampled episodes
per iteration $E$.
\Ensure Trained collusive adversarial policy network $\pi^{M}$ for CMA/DMA/SMA
modes.

\State Initialize $\pi^{M}$, $V(s)$, and CMA/DMA/SMA modules.
\State Initialize replay buffers for DMA and SMA modules.

\For{$k = 1, 2, \ldots, K$}
    \State Initialize an RL replay buffer $\mathcal{B}$.
    \For{$e = 1, 2, \ldots, E$}
        \State Reset the environment and the hidden states.
        \While{the current episode is not terminated}
            \State Obtain observations $o_t^{N}$ and $o_t^{M}$.
            \State Compute fused  $\tilde{o}_t$ via CMA module.
            \State Good agents sample normal actions $\mathbf{a}_t^{N} \sim \pi^{N}$.
            \State Malicious agents sample candidate adversarial actions $\mathbf{a}_t^{M} \sim \pi^{M}$.
            \State Construct the joint action $\mathbf{a}_t^{\text{union}} = (\mathbf{a}_t^{N}, \mathbf{a}_t^{M})$.
            \State Execute $\mathbf{a}_t^{\text{union}}$ in the environment and receive adversary reward $r_t^{M}$ and next state $s_{t+1}$.
            \State Store $(s_t, \tilde{o}_t, \mathbf{a}_t^{M}, r_t^{M}, s_{t+1})$ in replay buffer $\mathcal{B}$.
            \State Compute value increment $\Delta V_t = V(s_t) - V(s_{t-1})$.
        \EndWhile
    \EndFor

    \State Update $\pi^{M}$ using PPO-style policy gradient.
    \State Update CMA via backpropagation.\Comment{CMA training}
   \State Sample  $(\tilde{o}_t, \Delta V_{t+1})$.\Comment{DMA/SMA training}
\EndFor
\end{algorithmic}
\end{algorithm}

\textbf{Execution Procedure.}
Algorithm~\ref{alg:execution_pipeline} presents the execution pipeline. At each
time step, the procedure first produces normal actions and candidate adversarial
actions. It then invokes adversarial mode--specific gating logic (CMA, DMA, or
SMA) to dynamically determine which agents execute adversarial actions at the
current step.

\begin{algorithm}[tbhp]
\caption{\textbf{Execution Procedure}}
\label{alg:execution_pipeline}
\begin{algorithmic}[1]
\Require Fixed normal-policy network $\pi^{N}$ for good agents; trained adversarial-policy network $\pi^{M}$; trained CMA/DMA/SMA modules; average gating threshold $\tau_{\text{avg}}$ (DMA); individual gating threshold $\tau_{\text{ind}}$ (SMA).
\Ensure Joint action sequence $\mathbf{a}^{\text{union}}$ under selected mode.

\State Reset the environment and the hidden states.
\While{the current episode is not terminated}
    \State Obtain observations $o_t^{N}$ and $o_t^{M}$.
    \State Compute fused $\tilde{o}_t$ via CMA module.
    \State Good agents sample normal actions $\mathbf{a}_t^{N} \sim \pi^{N}$.
    \State Malicious agents sample candidate actions $\mathbf{a}_t^{M} \sim \pi^{M}$.

    \If{mode = CMA}
        \State Execute adversarial actions for  malicious agents.

    \ElsIf{mode = DMA}
        \For{each malicious agent $i$}
            \State Compute the prediction $\Delta \hat{V}_{i,t} = g_{\psi}(\tilde{o}_{i,t})$.
        \EndFor
        \State Compute the average score $\Delta \bar{V}_t$.
        \If{$\Delta \bar{V}_t > \tau_{\text{avg}}$}
            \State Execute adversarial actions for the malicious.
        \Else
            \State Execute normal actions for  malicious agents.
        \EndIf

    \ElsIf{mode = SMA}
        \For{each malicious agent $i$}
            \State Compute $\Delta \hat{V}_{i,t} = g_{\psi}(\tilde{o}_{i,t})$.
            \If{$\Delta \hat{V}_{i,t} > \tau_{\text{ind}}$}
                \State Execute adversarial action $a_{i,t}^{M}$ \Comment{Group B}
            \Else
                \State Execute normal action $a_{i,t}^{N}$ \Comment{Group A}
            \EndIf
        \EndFor
    \EndIf

    \State Transition to the next state $s_{t+1}$.
\EndWhile
\end{algorithmic}
\end{algorithm}

\renewcommand{\arraystretch}{1.25}

\begin{table*}[t]
\centering
\caption{Adversary reward across different SMAC II maps (reported as mean ± standard deviation).}
\label{tab:attack_data}
\resizebox{1\textwidth}{!}
{
\begin{tabular}{
>{\centering\arraybackslash}m{1.6cm}
c | c c c c c | c c c
}
\hline
\textbf{Map} & $\mathbf{n}$ & \textbf{AMI} & \textbf{IMAP} & \textbf{Wu et al.} & \textbf{Guo et al.} & \textbf{Gleave et al.} & \textbf{CMA} & \textbf{DMA} & \textbf{SMA} \\
\hline

\multirow{3}{*}{\makecell[c]{1c3s6z\\vs\\1c3s5z}}
& 2 & 16.39 $\pm$ 0.21 & 15.81 $\pm$ 0.13 & 15.46 $\pm$ 0.03 & 16.08 $\pm$ 0.03
    & 15.73 $\pm$ 0.14 & \textbf{17.31 $\pm$ 0.20} & 17.11 $\pm$ 0.47 & 16.96 $\pm$ 0.49 \\
& 3 & 18.28 $\pm$ 0.03 & 16.96 $\pm$ 0.01 & 16.80 $\pm$ 0.09 & 17.37 $\pm$ 0.03
    & 16.75 $\pm$ 0.01 & \textbf{18.70 $\pm$ 0.06} & 18.51 $\pm$ 0.19 & 18.29 $\pm$ 0.12 \\
& 4 & 17.64 $\pm$ 0.56 & 18.24 $\pm$ 0.23 & 17.81 $\pm$ 0.31 & 17.79 $\pm$ 0.04
    & 18.17 $\pm$ 0.51 & \textbf{19.41 $\pm$ 0.05} & 18.93 $\pm$ 0.17 & 18.57 $\pm$ 0.25 \\

\hline

\multirow{3}{*}{8m}
& 2 & 13.29 $\pm$ 0.29 & 13.24 $\pm$ 0.12 & 12.60 $\pm$ 0.72 & 13.07 $\pm$ 0.47
    & 12.93 $\pm$ 0.03 & \textbf{13.68 $\pm$ 0.04} & 13.44 $\pm$ 0.08 & 13.38 $\pm$ 0.06 \\
& 3 & 15.29 $\pm$ 0.11 & 15.26 $\pm$ 0.07 & 14.78 $\pm$ 0.61 & 15.49 $\pm$ 0.02
    & 14.28 $\pm$ 1.41 & \textbf{15.58 $\pm$ 0.03} & 15.34 $\pm$ 0.15 & 13.95 $\pm$ 1.02 \\
& 4 & 16.69 $\pm$ 0.39 & 16.86 $\pm$ 0.11 & 15.77 $\pm$ 1.09 & 16.17 $\pm$ 1.24
    & 16.68 $\pm$ 0.01 & \textbf{17.07 $\pm$ 0.02} & 16.79 $\pm$ 0.02 & 16.37 $\pm$ 0.09 \\

\hline

\multirow{3}{*}{MMM}
& 2 & 15.34 $\pm$ 0.04 & 15.51 $\pm$ 0.41 & 13.77 $\pm$ 0.19 & 14.50 $\pm$ 0.82
    & 15.22 $\pm$ 0.99 & \textbf{15.66 $\pm$ 0.09} & 15.27 $\pm$ 0.02 & 15.16 $\pm$ 0.08 \\
& 3 & 17.75 $\pm$ 0.03 & 17.74 $\pm$ 0.08 & 17.15 $\pm$ 0.43 & 17.27 $\pm$ 0.24
    & 16.28 $\pm$ 0.16 & \textbf{18.01 $\pm$ 0.06} & 17.10 $\pm$ 0.06 & 17.01 $\pm$ 0.29 \\
& 4 & 19.14 $\pm$ 0.02 & 18.80 $\pm$ 0.13 & 18.29 $\pm$ 0.98 & 17.58 $\pm$ 2.16
    & 18.87 $\pm$ 0.00 & \textbf{19.16 $\pm$ 0.01} & 18.91 $\pm$ 0.12 & 18.62 $\pm$ 0.06 \\

\hline

\multirow{3}{*}{1c3s5z}
& 2 & 15.98 $\pm$ 0.02 & 15.78 $\pm$ 0.09 & 15.32 $\pm$ 0.23 & 16.02 $\pm$ 0.01
    & 15.49 $\pm$ 0.06 & \textbf{17.05 $\pm$ 0.01} & 16.87 $\pm$ 0.05 & 16.68 $\pm$ 0.18 \\
& 3 & 17.43 $\pm$ 0.06 & 16.50 $\pm$ 0.01 & 16.20 $\pm$ 0.06 & 17.12 $\pm$ 0.02
    & 16.27 $\pm$ 0.01 & \textbf{18.17 $\pm$ 0.08} & 17.31 $\pm$ 0.24 & 17.12 $\pm$ 0.04 \\
& 4 & 18.59 $\pm$ 0.01 & 17.08 $\pm$ 0.15 & 17.40 $\pm$ 0.09 & 18.39 $\pm$ 0.02
    & 16.95 $\pm$ 0.01 & \textbf{18.87 $\pm$ 0.16} & 18.47 $\pm$ 0.16 & 17.43 $\pm$ 0.18 \\

\hline
\end{tabular}
}
\end{table*}

\section{Experiments}

\subsection{Experimental Setup}

\textbf{Experiment Environment.}
We utilize the commonly used StarCraft II Multi-Agent Challenge (SMAC II) platform~\cite{samvelyan2019starcraft} as a
benchmark to assess the performance of our approach, in consideration that SMAC II
offers diverse StarCraft II micromanagement scenarios with high-dimensional
partially observable states, nontrivial team coordination, and explicit
win/loss feedback. To comprehensively inspect the three adversarial attack modes,
we adopt four representative maps varying in squad size, unit heterogeneity, and
tactical complexity. For each map, we follow the standard SMAC II setting, and time
is discretized into fixed decision steps. At each step, every agent receives a
local observation, including its own status and information about nearby allies
and enemies (e.g., health, relative positions/distances, and partial terrain
features). The action space consists of discrete actions, such as moving
(up/down/left/right), attacking, and stopping. The environment reward encourages
maximizing enemy eliminations, minimizing allied casualties, and ultimately
winning the battle, which forms a canonical c-MARL task.

To obtain stable policies for good agents, we first pretrain
a MAPPO policy on each map without any adversarial interference. During training,
good agents share policy parameters and update using local observations with
access to the public global state. After convergence, we freeze the parameters
and use the resultant policy for all subsequent experiments.

Upon the backbone, we build upon a popular policy-level
single-adversary scheme~\cite{li2025attacking} and extend it to the three collusive adversarial variants
defined in this work. In our attack modes, malicious agents are a small subset
embedded within the team rather than additional external entities, while the good
agents always execute the frozen MAPPO policy.

\textbf{Baselines.}
The proposed three attack modes share the same backbone design, network
architecture, and optimization objective. The only difference lies in collusive
coordination and decision-making behaviors. To comparatively analyze the gain of
our proposed collusive adverseness, we compare our work with several policy-level
adversarial baselines: AMI~\cite{li2025attacking}, IMAP~\cite{zheng2024toward}, Wu et al.~\cite{wu2021adversarial}, Guo et al.~\cite{guo2021adversarial}, and Gleave et al.~\cite{gleave2020adversarial}. For a fair comparison, we extend
these baselines, originally designed for a single-adversary attack, to a
multi-adversary setting, i.e., multiple malicious agents exist, wherein each
malicious agent receives only its own local observation and independently
outputs adversarial actions without collusive behaviors. 
For consistency across different maps, we set $\tau_{\mathrm{avg}} = \tau_{\mathrm{ind}} = -0.30$ on the 1c3s6z\_vs\_1c3s5z and 8m maps, and $\tau_{\mathrm{avg}} = \tau_{\mathrm{ind}} = -0.40$ on the MMM and 1c3s5z maps.
During evaluation, we adopt adversary reward as the main metric to reflect attack strength, by which the
team performance of good agents can be degraded, while the attackers' return can
be improved; that is to say, the higher the adversary reward, the more successful
the adversarial policy.

\subsection{Attack Effect}
To obtain stable performance estimates, we randomly select five random seeds to execute
experiments for the three attack modes and baselines. The results are reported
in Table~\ref{tab:attack_data}, from which we can summarize the following observations:

\textbf{CMA Achieves Highest Adversary Reward across Four Maps.}
On the map 1c3s6z\_vs\_1c3s5z, CMA improves the strongest baseline by 5.6\%, 2.3\%,
and 6.4\% as the numbers of malicious agents are 2, 3, and 4, respectively, which
implies that the shared observation information among malicious agents can indeed
boost the adversary reward because of the collusive behaviors of multiple
malicious agents. As for the other maps, our approach also has the best
performance compared with the baselines. Compared to CMA, the other two attack
modes, DMA and SMA, behave slightly weaker due to the consideration of
stealthiness; however, both of them can still gain large adversary rewards as
well.

\textbf{Adversary Reward Increases with Number of Malicious Agents.}
Apart from our approach, the baseline methods almost exhibit the tendency that the more
the malicious agents, the more adversary rewards would be gained. A notable
exception appears on map 1c3s6z\_vs\_1c3s5z; that is, for the baseline AMI, when the number of malicious agents increases from 3 to 4, the adversary reward inversely
drops from means 18.2776 to 17.6425. This indicates that, without collusion,
independently malicious agents might interfere with each other, which perturbs
the adversarial learning and finally degrades the performance.

\textbf{Collusive Adversarial Attacks Accommodate More Sophisticated Maps.}
For the simple and more structured maps 8m and MMM, CMA yields small improvements
of adversary reward over baselines, i.e., 0.07\%-2.94\%. However, for the
sophisticated maps 1c3s6z\_vs\_1c3s5z and 1c3s5z with mixed unit types and formation
structures, CMA promotes over the strongest baselines by a larger margin, i.e.,
1.48\%-6.47\%. The results reflect that our proposed collusive attack mode can
validly accommodate more sophisticated c-MARL scenarios to gain more adversary
reward.

\renewcommand{\arraystretch}{1.2}

\begin{table*}[t]
\centering
\caption{Exposure Intensity across different SMAC II maps.}
\label{tab:detection_data}

\begin{tabular}{
>{\centering\arraybackslash}m{1.6cm}
c | c c c c c | c c c
}
\hline
\textbf{Map} & $\mathbf{n}$ & \textbf{AMI} & \textbf{IMAP} & \textbf{Wu et al.} & \textbf{Guo et al.} & \textbf{Gleave et al.} & \textbf{CMA} & \textbf{DMA} & \textbf{SMA} \\
\hline

\multirow{3}{*}{\makecell[c]{1c3s6z\\vs\\1c3s5z}}
& 2 & 1.57 & 1.27 & 1.36 & \textbf{1.20} & 1.21 & 1.58 & 1.38 & 1.29 \\
& 3 & 2.11 & 1.84 & 1.80 & 1.46 & 1.71 & 2.13 & 1.62 & \textbf{1.41} \\
& 4 & 2.45 & 2.03 & 2.21 & 1.90 & 2.20 & 2.17 & 1.92 & \textbf{1.45} \\

\hline

\multirow{3}{*}{8m}
& 2 & 1.18 & 1.41 & 1.44 & 0.95 & \textbf{0.83} & 1.12 & 1.00 & 0.93 \\
& 3 & 1.69 & 1.81 & 1.68 & 1.45 & 1.68 & 1.28 & 1.00 & \textbf{0.96} \\
& 4 & 1.93 & 2.09 & 2.20 & 1.76 & 1.97 & 2.03 & 1.78 & \textbf{1.67} \\

\hline

\multirow{3}{*}{MMM}
& 2 & 1.60 & 1.61 & 1.62 & 1.41 & 1.71 & 1.64 & 1.56 & \textbf{1.31} \\
& 3 & 2.32 & 1.81 & 1.77 & 1.50 & 1.89 & 2.30 & 1.69 & \textbf{1.49} \\
& 4 & 2.99 & \textbf{1.72} & 2.01 & 1.73 & 1.99 & 3.03 & 2.29 & 2.07 \\

\hline

\multirow{3}{*}{1c3s5z}
& 2 & 1.52 & 0.98 & \textbf{0.95} & 0.98 & 1.07 & 1.41 & 1.17 & 0.96 \\
& 3 & 1.73 & 1.23 & 1.29 & 1.15 & 1.32 & 1.82 & 1.15 & \textbf{0.99} \\
& 4 & 2.10 & 1.38 & 1.54 & 1.41 & 1.63 & 2.05 & 1.56 & \textbf{1.23} \\

\hline
\end{tabular}

\end{table*}

\subsection{Stealthiness Evaluation}
As we know, if an attack is overly aggressive or exhibits visibly abnormal behaviors, it may achieve high short-term rewards but can be quickly detected and thus become ineffective. Therefore, stealthiness and long-term sustainability are also key criteria for evaluating the adversarial behaviors of malicious agents.

To this end, we analyze stealthiness using a unified action-distribution-based detector, namely Tracker~\cite{kazari2023decentralized}, as a standardized anomaly detection module. During evaluation, the Tracker outputs an anomaly score for each agent at every time step and aggregates the scores using a sliding window mechanism.

\textbf{Setup.}
Given a window length $W$ and a detection threshold $\tau_{\mathrm{det}}$, we quantify stealthiness using \emph{Exposure Intensity} as the evaluation metric. For a single episode, Exposure Intensity is defined as the ratio between the total number of detected events triggered by malicious agents and the episode length:
\begin{equation*}
\mathrm{Exposure\ Intensity}
=
\frac{\sum_{i \in \mathcal{M}} C_i}{T},
\end{equation*}
where $\mathcal{M}$ denotes the set of malicious agents, $C_i$ is the number of detection events associated with agent $i$, and $T$ is the total number of time steps in the episode. This metric characterizes how frequently adversarial behaviors are exposed per unit time. We report the average Exposure Intensity over multiple evaluation episodes.
Unless otherwise specified, we adopt the default Tracker configuration with window length $W = 3$ and detection threshold $\tau_{\mathrm{det}} = -0.3$. 

\textbf{Exposure Intensity.}
In adversarial behavior detection, the conventional detection rate is usually defined as the proportion of detected time steps over the total episode length. However, this assumption often fails in collusive adversarial attack scenarios. As shown in Appendix~\ref{app:D}, changes in attacks can significantly affect the episode duration. Stronger or more concentrated attacks may terminate an episode much earlier.

Under such conditions, the conventional detection rate can be systematically biased. Even when adversarial behaviors trigger anomalies more frequently per unit time, the episode may end earlier. When normalized by a shorter episode length, the overall detection rate can paradoxically decrease.

In contrast, Exposure Intensity focuses on the temporal density of detected adversarial behaviors. It accumulates the number of anomaly triggers caused by malicious agents within an episode and explicitly normalizes them over time. This formulation treats detection as an intensity rather than a proportion. As a result, it avoids misinterpreting early episode termination as higher stealthiness and enables a fairer comparison of different attacks on a unified temporal scale.

Table~\ref{tab:detection_data} reports the Exposure Intensity results on four maps with the number of malicious agents ranging from 2 to 4, from which we draw the following conclusions.

\textbf{Both DMA and SMA Reduce Exposure Intensity.}
In most settings, they achieve lower Exposure Intensity, which indicates that
collusive attack modes with value-driven attack triggering and role
(Collector/Executor) differentiation can substantially deceive the detector,
leading to high stealthiness and long-term sustainability. Comparatively, SMA in
general behaves more stealthily than DMA across all maps. For instance, on the
8m map with three malicious agents, the best baseline achieves an Exposure
Intensity of 1.45, whereas DMA and SMA reduce it to 1.00 and 0.96,
respectively. In particular, SMA yields an Exposure Intensity reduction of
approximately 33.5\% compared to the baseline.

\textbf{Exposure Intensity Follows A Stable Ordering across Three
Attack Modes: CMA $>$ DMA $>$ SMA.}
The results well match our previous theoretical analysis. CMA emphasizes attack
effectiveness while disregarding stealthiness and long-term sustainability. In
contrast, DMA and SMA explicitly target stealthiness while simultaneously
maintaining strong attack effects, although weaker than CMA.

\textbf{Exposure Intensity Increases as Number of Malicious Agents
Enlarges.}
More malicious agents cause an increase in Exposure Intensity, stemming from the
accumulated misbehaviors across malicious individuals. Consequently, anomaly signals are more likely to be captured by the action-distribution-based detector.

\subsection{Ablation Studies}
\subsubsection{Ablations on CMA}
To validate the necessity and effectiveness of CMA, we conduct ablation studies from two aspects: information sharing and structured modeling. Specifically, on the map 1c3s6z\_vs\_1c3s5z, we keep the training pipeline, network size, and hyperparameters unchanged, and only vary how malicious agents interact w.r.t. shared observation information. We compare CMA with the following two variants:


\begin{figure*}[t]
    \centering
    \begin{subfigure}[t]{0.65\columnwidth}
        \centering
        \includegraphics[width=\linewidth]{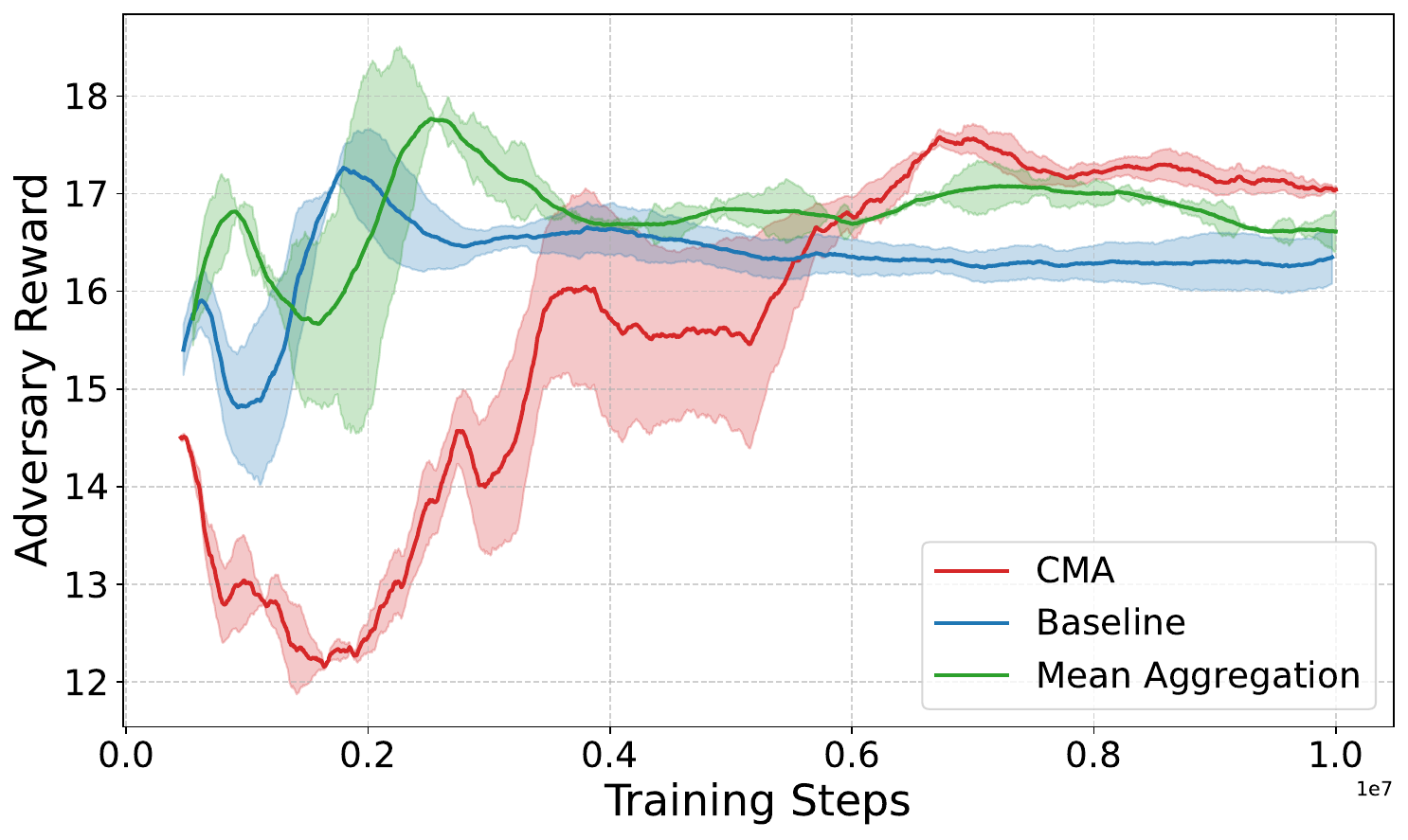}
        \caption{Reward with 2 malicious agents}
        \label{fig:ablation_cma_2}
    \end{subfigure}
    \hfill
    \begin{subfigure}[t]{0.65\columnwidth}
        \centering
        \includegraphics[width=\linewidth]{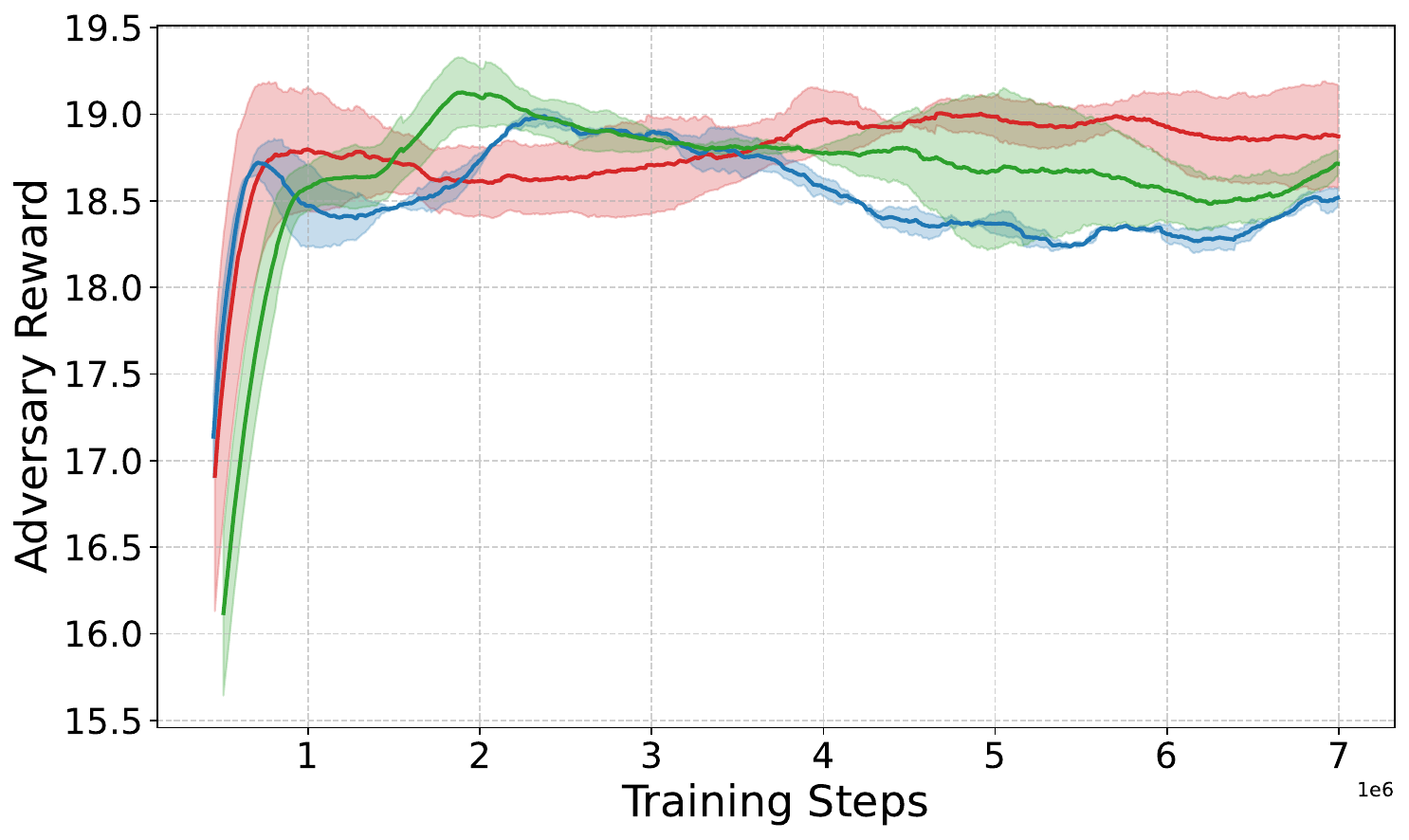}
        \caption{Reward with 3 malicious agents}
        \label{n=3}
    \end{subfigure}
    \hfill
    \begin{subfigure}[t]{0.65\columnwidth}
        \centering
        \includegraphics[width=\linewidth]{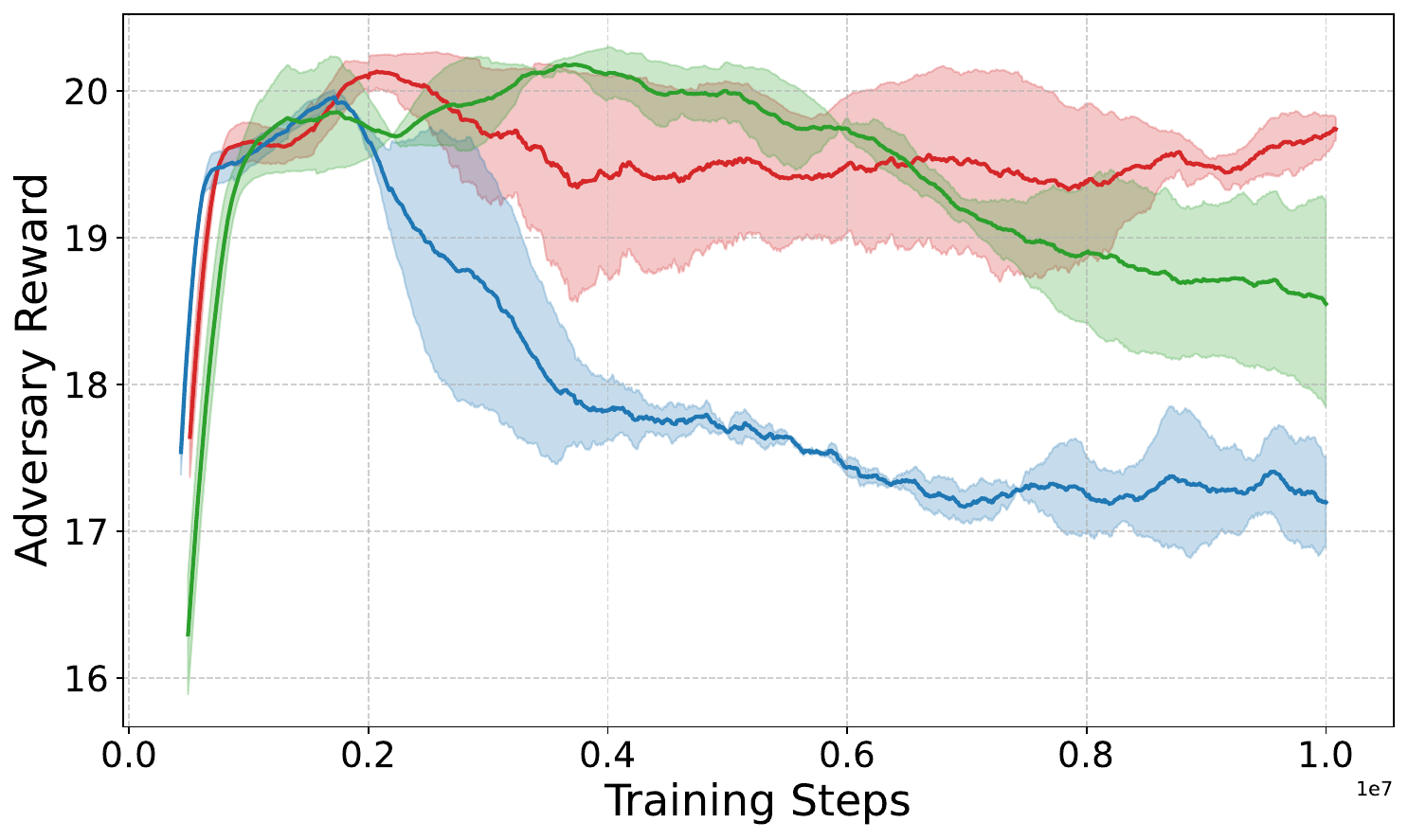}
        \caption{Reward with 4 malicious agents}
        \label{n=4}
    \end{subfigure}

    \caption{Training curves of average adversary reward under different numbers of malicious agents. Solid line denotes the mean performance over multiple runs, while shaded region indicates the corresponding standard deviation.}
    \label{fig:ablation_cma}
\end{figure*}

\textbf{Baseline (AMI \cite{li2025attacking}).}
Malicious agents do not share observation information, that is to say, each malicious agent makes decisions using only its own local observation as input to its policy network. This setting cannot exploit the collusive structure among multiple malicious agents.

\textbf{Mean Aggregation.}
Malicious agents are allowed to share observations, but the shared information is fused using a simple unstructured mean aggregation. This variant enables information sharing but removes structured modeling.

Fig.~\ref{fig:ablation_cma} shows the adversary reward, from which we can see that CMA consistently achieves higher adversary reward than Baseline and Mean Aggregation across all group sizes. The better perfo rmance of Mean Aggregation over Baseline stems from enabling information sharing, and the superiority of CMA over Mean Aggregation variant stems from its ability to fully capture the latent collusive structure among all malicious agents through a Transformer-based fusion mechanism, rather than simple mean aggregation.

\subsubsection{Ablations on DMA}
\label{ablations on dma}

In this ablation, we only change how an attack is triggered at each time step. We design different temporal gating mechanisms and compare the following two variants:

\textbf{Always-On Attack (CMA).}
Each malicious agent executes an adversarial action at every step.

\textbf{Random Triggering.}
At each time step, the attack is triggered with a fixed probability $p$; with probability $1 - p$, all malicious agents keep the normal action. We tune $p$ via a small pilot study so that the average attack frequency is approximately matched to DMA on the same map. 

For DMA, the threshold is fixed at $\tau_{\text{avg}} = -0.30$ and
$-0.40$ for 1c3s6z\_vs\_1c3s5z and 1c3s5z, respectively.
Fig.~\ref{fig:ablation_dma} \, shows the comparison of three triggering mechanisms on the 1c3s6z\_vs\_1c3s5z and 1c3s5z map. CMA achieves the highest adversary reward under all settings, but it also leads to significantly higher exposure intensity. Since the attack frequencies are comparable, DMA and Random Trigger exhibit similar exposure intensity. However, DMA consistently outperforms Random Trigger in terms of adversary reward. These results indicate the advantage of DMA does not come from simply reducing the attack frequency. Instead, it relies on a value-driven attack triggering mechanism, which enables higher attack effectiveness under similar exposure levels.

\begin{figure}[t]
    \centering
    \begin{subfigure}[t]{0.49\columnwidth}
        \centering
        \includegraphics[width=\linewidth]{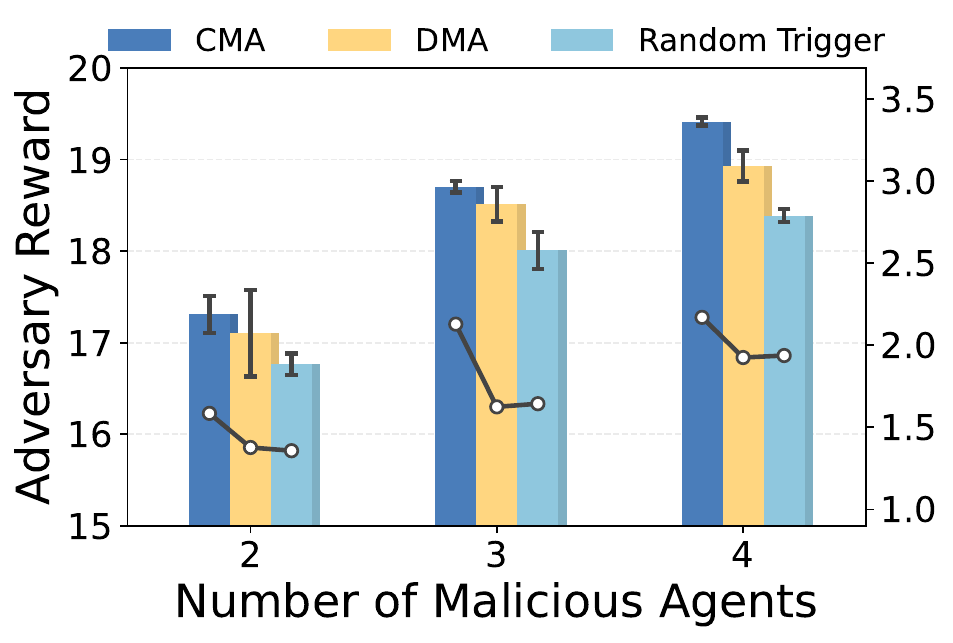}
        \caption{1c3s6z\_vs\_1c3s5z}
        \label{fig:ablation_dma_map1}
    \end{subfigure}
    \hfill
    \begin{subfigure}[t]{0.49\columnwidth}
        \centering
        \includegraphics[width=\linewidth]{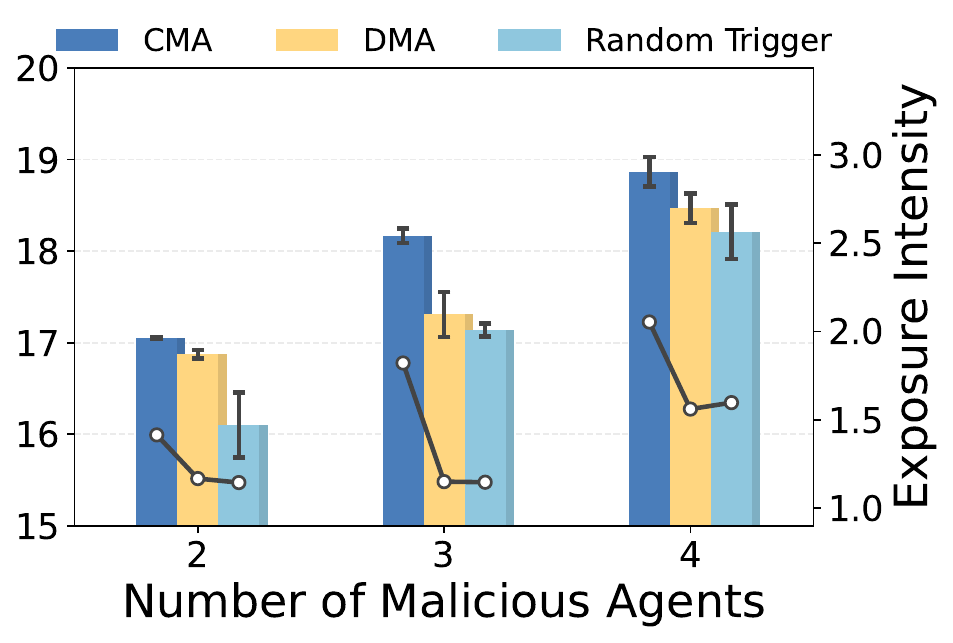}
        \caption{1c3s5z}
        \label{fig:ablation_dma_map2}
    \end{subfigure}

    \caption{Ablation results of different triggering mechanisms on SMAC II maps. Bars indicate adversary reward, and lines indicate exposure intensity.}
    \label{fig:ablation_dma}
\end{figure}

\subsubsection{Ablations on SMA}
In this ablation, we only change how roles are assigned among malicious agents. We build the following control variant:

\textbf{Random Grouping.}
At each time step, each malicious agent is independently assigned to Group B (Executors) with probability $p$, and to Group A (Collectors) with probability $1-p$.
The probability $p$ is tuned via a small pilot study such that the resulting average attack frequency is approximately matched to that of SMA on the same map.

For SMA, we fix the threshold $\tau_{\mathrm{ind}}$ following the same setting as in Section~\ref{ablations on dma}.
As shown in Fig.~\ref{fig:ablation_sma}, Random Group and SMA exhibit similar exposure intensity. However, in terms of attack effectiveness, SMA consistently outperforms Random Group across all settings, achieving higher adversary reward. These results indicate the advantage of SMA does not stem from simple random role assignment, but from its value-driven dynamic role allocation, which enables more effective coordination among malicious agents while simultaneously maintaining stealthiness.

\begin{figure}[t]
    \centering
    \begin{subfigure}[t]{0.49\columnwidth}
        \centering
        \includegraphics[width=\linewidth]{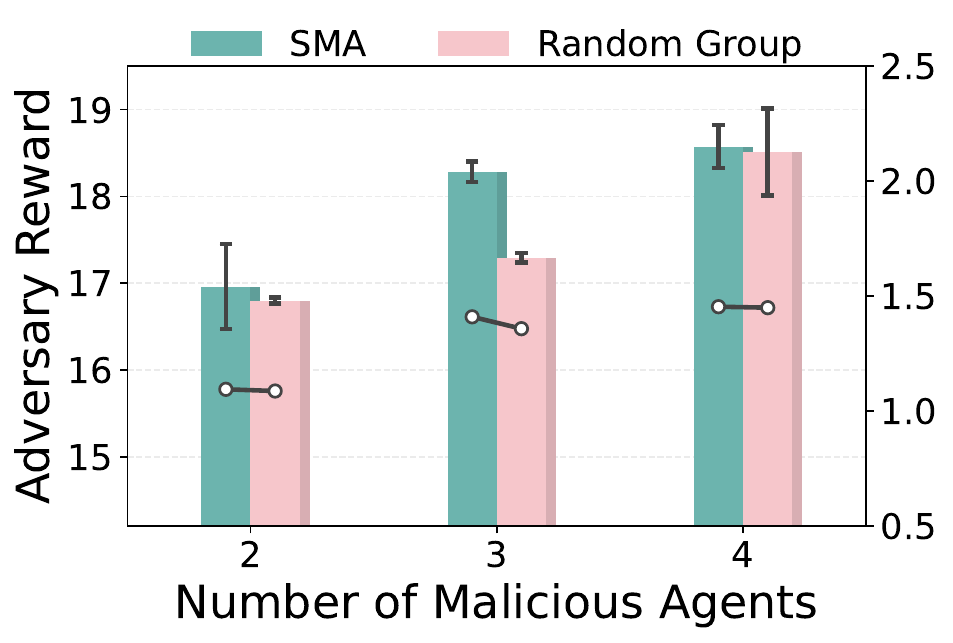}
        \caption{1c3s6z\_vs\_1c3s5z}
        \label{fig:ablation_sma_map1}
    \end{subfigure}
    \hfill
    \begin{subfigure}[t]{0.49\columnwidth}
        \centering
        \includegraphics[width=\linewidth]{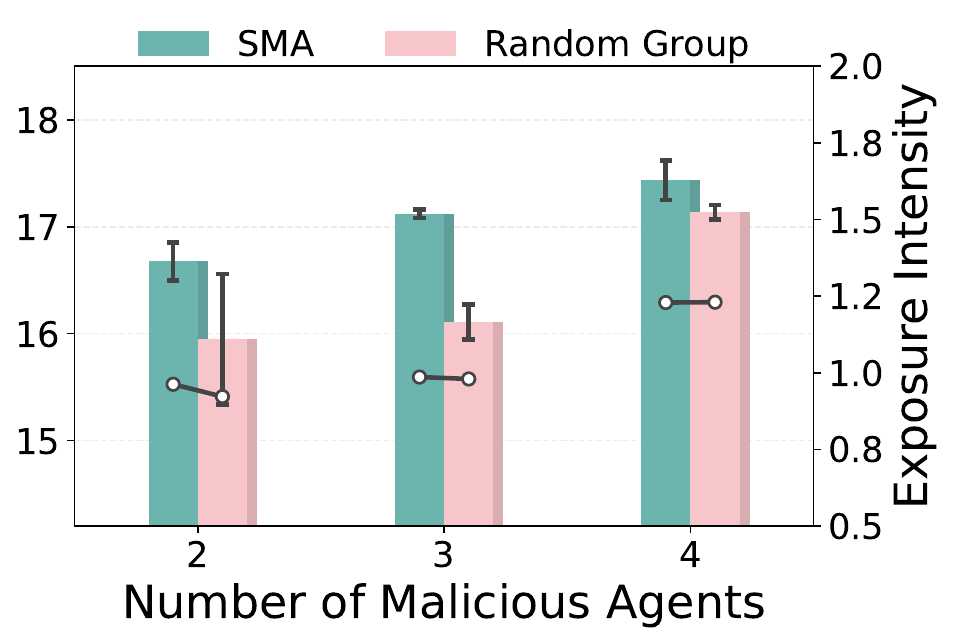}
        \caption{1c3s5z}
        \label{fig:ablation_sma_map2}
    \end{subfigure}

    \caption{Ablation results of different grouping mechanisms on SMAC II maps.
    Bars indicate adversary reward, and lines indicate exposure intensity.}
    \label{fig:ablation_sma}
\end{figure}
 
\subsubsection{Sensitivity to DMA and SMA}
In our CAMA framework, both DMA and SMA rely on a gating threshold to determine when adversarial actions are activated. 
As a result, the gating threshold directly controls the attack frequency and strength, while also influencing the temporal sparsity of perturbations and their detectability. 
To systematically analyze its impact, we vary the gating threshold while keeping all other modules and experimental settings fixed, and evaluate the performance of DMA and SMA under different threshold values.

\textbf{Setup.}
We conduct experiments on the 1c3s6z\_vs\_1c3s5z map with two malicious agents, using the same adversarial training pipeline for all settings. 
The gating thresholds $\tau_{\mathrm{avg}}$ (for DMA) and $\tau_{\mathrm{ind}}$ (for SMA) are varied from $-0.30$ to $0$. 
For each threshold setting, we report the average Adversary Reward, Exposure Intensity, and Attack Frequency.

As illustrated in Fig.~\ref{fig:threshold_sensitivity}, as gating threshold
$\tau_{\text{det}}$ increases, both DMA and SMA exhibit consistent trends:
as the gating threshold increases, the adversary reward steadily decreases, while both the exposure intensity and attack frequency are reduced.
This trend indicates that the gating threshold effectively controls the temporal sparsity of adversarial attacks and induces a clear trade-off between attack strength and stealthiness.
Smaller thresholds allow more frequent adversarial activations, leading to stronger attack effectiveness but higher exposure risk.
In contrast, larger thresholds suppress attack triggering, resulting in sparser adversarial behaviors with lower detectability, at the cost of reduced attack performance.

\begin{figure}[t]
    \centering
    \includegraphics[width=2.7in, height=2.4in]
    {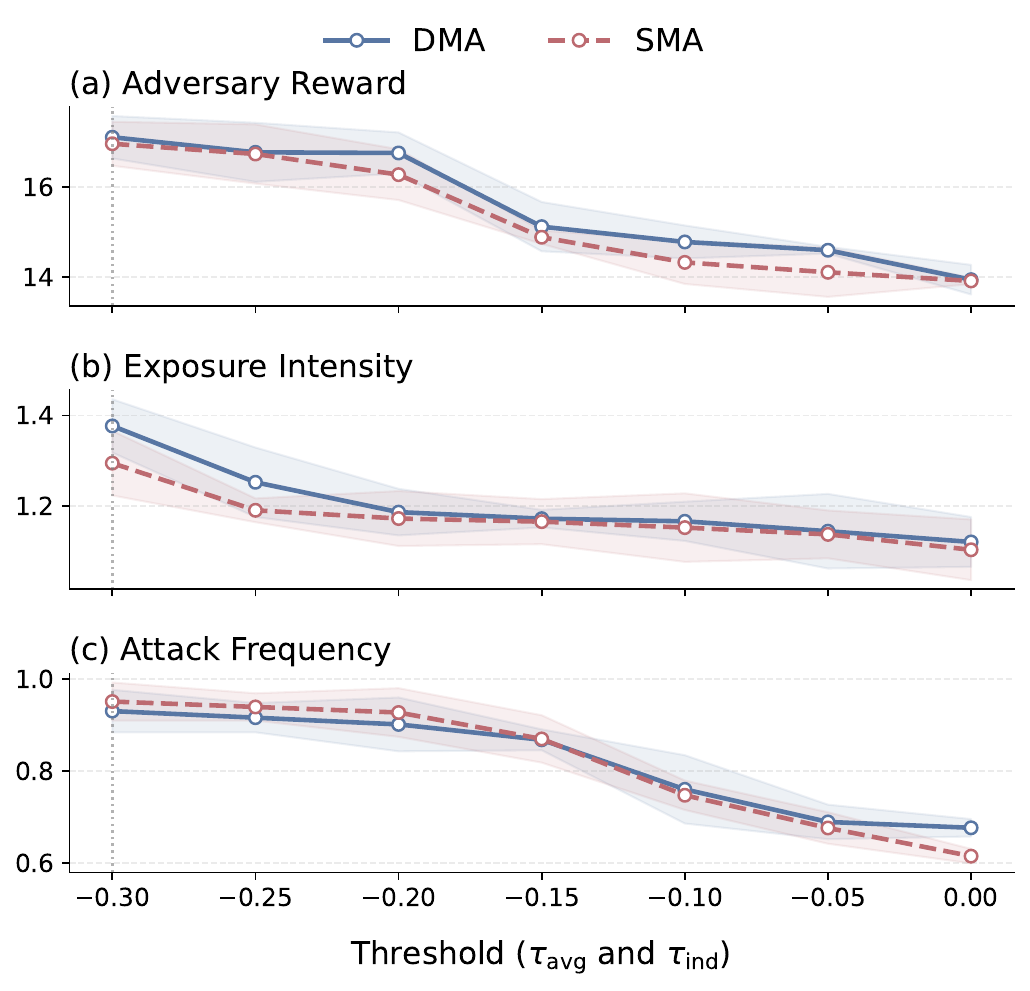}
    \caption{Sensitivity of DMA and SMA w.r.t. gating threshold.} 
    \label{fig:threshold_sensitivity}
\end{figure}

\subsubsection{Sensitivity to the Detector}
The Tracker’s window length and detection threshold jointly control how fast and how robustly it responds to abnormal trajectories. A longer window gives more stable statistics but slower response. A lower threshold makes detection more sensitive but increases false alarms. To study this effect, we conduct a sensitivity analysis of the Tracker’s parameters.

\textbf{Setup.} We fix the attack mode to CMA on the 1c3s6z\_vs\_1c3s5z map with three malicious agents, and vary only the Tracker configurations on the defense side. Specifically, we evaluate window lengths $W \in \{1, 3, 5, 10, 15, 20, 30, -1\}$ 
and detection thresholds $\tau_{\mathrm{det}}$ ranging from $-6$ to $1$. 
Here, $W = -1$ denotes the use of a global window without fixed-length truncation, where the Tracker aggregates statistics over the entire episode history.

Fig.~\ref{fig:exposure_density} illustrates the Exposure Intensity across different window lengths and detection thresholds, showing how these parameters influence the detection dynamics, from which we can conclude that, with a fixed window length, increasing the detection threshold \(\tau_{\text{det}}\) makes the Tracker shift from conservative to aggressive. Exposure intensity rises quickly and becomes saturated. This pattern is consistent across different window lengths. A higher threshold leads to faster and more complete abnormality detection. With a fixed detection threshold, increasing the window length \( W \) slows down detection and reduces exposure. This unveils that a longer window smooths short-term abnormal spikes and reduces false alarms. When using a global window, this trend becomes stronger. Under low thresholds, detection is almost suppressed, showing that aggregation without time truncation greatly reduces sensitivity to local anomalies.

In addition to Exposure Intensity, we further report the First Detection Time and Detection Rate under the same settings; detailed results are provided in Appendix~\ref{app:G}.

\begin{figure}[t]
    \centering
    \includegraphics[width=2.5in, height=1.2in]{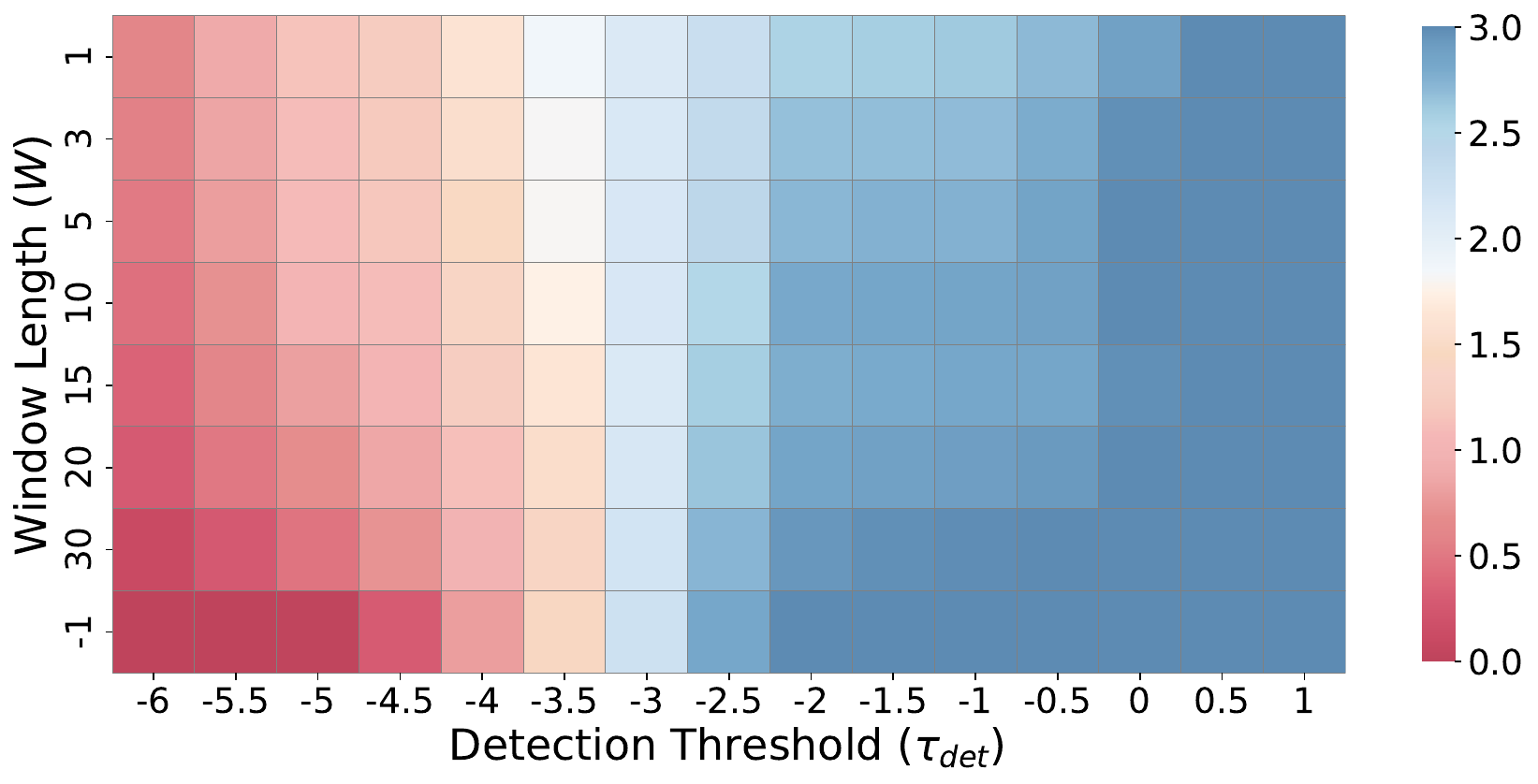}  
    \caption{Exposure Intensity across different window lengths and detection threshold.} 
    \label{fig:exposure_density}
\end{figure}


\section{Related Work}
\subsection{Adversarial Attacks in RL}
Adversarial attacks were first studied in computer vision, where the central idea is to inject small perturbations into input samples that are almost imperceptible to humans yet can substantially alter a model’s prediction~\cite{szegedy2013intriguing,goodfellow2014explaining,carlini2017towards}. With the broad adoption of Reinforcement Learning (RL) in control and decision-making domains, scholars have recognized that RL-based agents are similarly susceptible to adversarial perturbations~\cite{huang2017adversarial,kos2017delving}. Unlike static tasks such as image classification, RL involves sequential decision-making through continual interaction with the environment, i.e., the agent learns in a "State-Action-Reward" loop over time. Consequently, an attacker is not restricted to perturbing a single observation; instead, it can systematically mislead the policy across multiple time steps by dynamically perturbing the agent’s observations, reward feedback, or environment dynamics (e.g., state transitions), ultimately yielding a significant degradation in long-term cumulative return. For example, in game-like environments~\cite{rakhsha2020policy}, forged rewards or delayed feedback can bias learning toward incorrect long-horizon policies.

Adversarial attacks in RL are commonly categorized into two classes: training-time attacks and test-time attacks. The former occurs during learning, where the adversary alters the optimization target by tampering with training data, manipulating reward signals, or inserting covert triggers. Typical forms include data poisoning~\cite{huang2019deceptive,rakhsha2020policy,wu2022copa}, aiming to shift the training distribution or inject malicious samples to impair generalization, and backdoor attacks~\cite{behzadan2019sequential,panagiota2020trojdrl,wang2021backdoorl,yang2019design}, embedding triggers to enforce the agent to execute adversary-specified actions under particular inputs. In contrast, the latter occurs during deployment or inference: the adversary does not modify model parameters but instead misleads a trained policy by manipulating external interactions—such as observations, the environment, or other agents’ behaviors. This line of work aligns more closely with realistic threat models, since attackers in practice rarely have the capability to retrain the target agent but may still exert influence at runtime. Since our work focuses on policy-level collusive adversarial attacks, it falls into the scope of test-time attacks.

\subsection{Observation-/Action-Level Attacks}
In test-time adversarial attacks, early studies predominantly focus on two forms of direct perturbations: observation perturbations and action injection. An adversary can tamper with an agent’s observation input, such as adding small noise to the state vector, masking salient features, or crafting adversarial images, thus inducing the policy network to output incorrect actions~\cite{huang2017adversarial,qiaoben2024understanding}. Another line of attacks perturbs the decision output by replacing the agent’s selected action with an adversary-specified action, or injecting random or structured noise into the action space~\cite{lee2020spatiotemporally}.
However, these attacks often rely on strong white-box assumptions: the adversary must have access to the victim agent’s model architecture and gradients, and be able to directly manipulate the observation signal or the action interface. In realistic deployments, such privileges are almost unavailable, limiting the transferability and practical applicability.

\subsection{Policy-Level Attacks}
To conquer the unrealistic assumptions, recent research has increasingly shifted toward policy-level attacks. Instead of directly modifying inputs, this line of work introduces additional adversarial agents into the environment and indirectly misleads the victim's decision-making via policy interactions.

Gleave et al.~\cite{gleave2020adversarial} first introduced the notion of adversarial policies in two-player zero-sum games. They trained a malicious agent whose actions are able to perturb the victim’s observations, thereby inducing irrational behaviors. Unlike conventional observation-level perturbations, this attack does not require access to the victim’s model parameters; rather, it realizes perturbations via policy interaction, substantially improving practical feasibility. Building on this paradigm, subsequent studies enhance the effectiveness of adversarial policies. Wu et al.~\cite{wu2021adversarial} leverage saliency analysis to identify observation features that are most sensitive and then perform targeted interference. Guo et al.~\cite{guo2021adversarial} extend adversarial policies to more general zero-sum settings by jointly optimizing the reward objectives of the attacker and victim.

Despite promising results in two-agent game scenarios, most existing approaches remain confined to an adversarial game formulation with a single victim or a narrowly scoped objective—typically degrading one agent’s performance or reducing the opponent’s payoff—without considering system-level disruption in multi-agent systems. To fill this gap, Li et al.~\cite{li2025attacking} introduced adversarial policies into c-MARL and proposed the first policy-level attack tailored to c-MARL, achieving solid empirical results. However, their formulation assumes only a single malicious agent.
This assumption overlooks a more representative real-world scenario. For instance, in many complex systems, malicious behaviors are rarely isolated; instead, they often emerge as organized actions by multiple attackers through collusive misbehavior. Thus, collusive adversarial attacks in c-MARL remain unexplored. To fill this gap, we design and analyze three policy-level collusive attack modes. Compared with prior work primarily focusing on \textit{single-attacker–single-victim} or \textit{single-attacker–multiple-victims} settings, our work characterizes \textit{multi-attacker coordination} at both policy level and organization level.

\section{Conclusion}
This paper innovatively explores collusive adversarial attacks for cooperative multi-agent reinforcement learning. First, we creatively define three collusive attacks regarding information sharing, attack-trigger control, and role assignment. Second, we model the three attack modes and theoretically analyze the attack effect, stealthiness, and attack cost. Third, we conduct multi-faceted experiments using the popular SMAC~II benchmark platform to validate the effectiveness of our proposed collusive attacks, and the strong consistency between experimental results and theoretical analysis demonstrates the validity of CMA, DMA, and SMA attacks in c-MARL scenarios. We believe our study opens the door to collusive adversarial attack research in multi-agent systems, and simultaneously pave the base for future defense research. 


\cleardoublepage
\appendix
\cleardoublepage
\bibliographystyle{plain}
\bibliography{references}

\appendix
\section{Proof of Theorem 1}
\label{app:A}

\begin{proof}
Under independent decision-making, the information available to malicious agent $i$ at time $t$ is its local observation:
$Y^{\text{ind}}_{i,t} = o_{i,t}$.
Under CMA attack, the available information is the shared aggregated
observation:
$Y^{\text{L1}}_{i,t} = \tilde{o}_t = F(\{o_{i,t}\}_{i \in \mathcal{M}})$.

According to Blackwell’s theory\cite{blackwell1953equivalent} of information structure comparison, if there exists a mapping $Y^{(2)} = \phi(Y^{(1)})$, then the information structure
$Y^{(1)}$ dominates the information structure $Y^{(2)}$, which is denoted as
$Y^{(1)} \succeq Y^{(2)}$.
This dominance relation means that, for any information-based decision problem, the maximum expected utility achievable using $Y^{(1)}$ is no smaller than that achievable using $Y^{(2)}$.

In our case, for any given agent $i$, its original local observation $o_{i,t}$ constitutes a part of the fused input, thus there exists a deterministic mapping function $g(\cdot)$ such that
$o_{i,t} = g_i(\tilde{o}_t)$.
This implies that $Y^{\text{ind}}_{i,t} = g_i(Y^{\text{L1}}_{i,t})$.
The mapping $g_i$ is deterministic. Hence, a stronger conclusion holds that
$Y^{\text{L1}}_{i,t} \succeq Y^{\text{ind}}_{i,t}$.
That is, the information structure under CMA is strictly no less informative
than that under independent attacks.

Define the adversary’s utility as the negative of the system return:
\begin{equation}
U(\Pi) := - \mathbb{E}_{\Pi}
\left[\sum_{t=0}^{T} \gamma^{t} R(s_t, a_t)\right].
\label{eq:adv_utility_app}
\end{equation}
In a finite-horizon Dec-POMDP, given any information structure $Y$, the policy set $\Pi(Y)$ denotes the collection of all feasible policies that depend solely on this information structure.
Given the globally-integrated information, malicious agents seek to maximize $U(\Pi)$.
Since $Y^{\text{L1}}_{i,t}$ is more informative than $Y^{\text{ind}}_{i,t}$, we have
$\sup U(\Pi^{\text{L1}}) \ge \sup U(\Pi^{\text{ind}})$, which implies
$J^{\text{L1}} \le J^{\text{ind}}$.
The normal return $J^{\text{normal}}$ is independent of the adversary’s
information structure.
Hence, we have
\begin{equation}
D_{\text{L1}} = J^{\text{normal}} - J^{\text{L1}}
\ge J^{\text{normal}} - J^{\text{ind}} = D_{\text{ind}}.
\label{eq:cma_vs_ind}
\end{equation}
\end{proof}

\section{Proof of Theorem 2}
\label{app:B}

\begin{proof}
Under CMA attack, malicious agents launch adversarial attacks at every time
step, i.e., $x_t^{\text{L1}} = 1$.
Under DMA, the adversarial perturbations are performed selectively with
$x_t^{\text{L2}} \in \{0,1\}$.
Define the set of time steps in which CMA executes attacks but DMA does not as
\begin{equation}
\mathcal{U} = \{\, t \mid x_t^{\text{L1}} = 1,\; x_t^{\text{L2}} = 0 \,\}.
\label{eq:U}
\end{equation}

Thus, the difference of disruptiveness can be calculated as
\begin{align}
D(x^{\text{L2}}) - D(x^{\text{L1}})
&= \sum_t g_t \bigl(x_t^{\text{L2}} - x_t^{\text{L1}}\bigr) \notag \\
&= - \sum_{t \in \mathcal{U}} g_t .
\label{eq:disruptiveness_diff}
\end{align}
The cumulative exposure under CMA is $\Delta^{\text{L1}} = \sum_t \delta_t$, whereas under DMA it becomes $\Delta^{\text{L2}} = \Delta^{\text{L1}} - \sum_{t \in \mathcal{U}} \delta_t$.
Since $\Delta^{\text{L2}} \le \Delta^{\text{L1}}$, as well as $S(\cdot)$ is a monotone non-increasing function and satisfies the lower bound in Eq.~\eqref{eq:lipschitz_lower}, we obtain
\begin{equation}
S(\Delta^{\text{L2}}) - S(\Delta^{\text{L1}})
\ge k(\Delta^{\text{L1}} - \Delta^{\text{L2}})
= k \sum_{t \in \mathcal{U}} \delta_t .
\label{eq:stealthiness_diff}
\end{equation}
Furthermore, the difference in attack cost is
\begin{equation}
C(x^{L2}) - C(x^{L1}) = \sum_t c_t (x^{L2}_t - x^{L1}_t) = -\sum_{t \in \mathcal{U}} c_t .
\label{eq:cost_diff}
\end{equation}
Therefore, the difference in AEF function can be written as
\begin{equation}
\begin{aligned}
J_{\text{L2}} - J_{\text{L1}}
&= \alpha \!\left[ D(x^{\text{L2}}) - D(x^{\text{L1}}) \right]
 + \beta \!\left[ S(\Delta^{\text{L2}}) - S(\Delta^{\text{L1}}) \right] \\
&\quad - \gamma \!\left[ C(x^{\text{L2}}) - C(x^{\text{L1}}) \right].
\end{aligned}
\label{eq:aef_expand}
\end{equation}

By substituting Eq.~\eqref{eq:disruptiveness_diff}, Eq.~\eqref{eq:stealthiness_diff} and Eq.~\eqref{eq:cost_diff} into Eq.~\eqref{eq:aef_expand}, we obtain
$J_{\mathrm{L2}} - J_{\mathrm{L1}}
\ge \sum_{t \in \mathcal{U}}
\left( - \alpha g_t + \beta k \delta_t + \gamma c_t \right)$.
Next, we define the per-step adversarial advantage for all malicious agents as
\begin{equation}
\Phi_t = \alpha g_t - \beta k \delta_t - \gamma c_t .
\label{eq:phi_def}
\end{equation}
The lower bound can be rewritten as
$J_{\mathrm{L2}} - J_{\mathrm{L1}} \ge \sum_{t \in \mathcal{U}} (-\Phi_t)$.

If DMA excludes all time steps with $\Phi_t \le 0$, then for every
$t \in \mathcal{U}$, $-\Phi_t > 0$ holds, which implies
$J_{\mathrm{L2}} - J_{\mathrm{L1}} \ge \sum_{t \in \mathcal{U}} (-\Phi_t) > 0$.
Hence, we have \[
J_{\text{L2}} > J_{\text{L1}}.
\]
If no ineffective time steps exist, i.e., $\mathcal{U} = \varnothing$, then we have
$x^{\text{L2}} = x^{\text{L1}}$ and $J_{\text{L2}} = J_{\text{L1}}$.
\end{proof}

\section{Proof of Theorem 3}
\label{app:C}
\begin{proof}
Let $x^{\text{L2}}_{i,t}$ and $x^{\text{L3}}_{i,t}$ denote the agent-level
adversarial triggers under DMA and SMA, respectively.
Define the set of agent-time pairs that would perturb adversarially under DMA but are reassigned to Group~A under SMA as
$\mathcal{U} = \{\, (i,t) \mid x^{\text{L2}}_{i,t} = 1,\; x^{\text{L3}}_{i,t} = 0 \,\}$.

Accordingly, the disruptiveness difference is given by
\begin{align}
D(x^{\text{L3}}) - D(x^{\text{L2}})
&= \sum_{i,t} g_{i,t} \left( x^{\text{L3}}_{i,t} - x^{\text{L2}}_{i,t} \right) \notag \\
&= - \sum_{(i,t)\in\mathcal{U}} g_{i,t}.
\label{eq:disrupt_dma_sma}
\end{align}
The difference in cumulative exposure can be written as
$\Delta^{\text{L3}} - \Delta^{\text{L2}}
= \sum_{i,t} \delta_{i,t}\!\left(x^{\text{L3}}_{i,t} - x^{\text{L2}}_{i,t}\right)
= - \sum_{(i,t)\in\mathcal{U}} \delta_{i,t}$, thus $\Delta^{\text{L3}} = \Delta^{\text{L2}} - \sum_{(i,t)\in\mathcal{U}} \delta_{i,t} \le \Delta^{\text{L2}}$.
Since $S(\cdot)$ is monotone non-increasing and satisfies the lower bound in Eq.~\eqref{eq:lipschitz_lower}, we have
\begin{equation}
S(\Delta^{\text{L3}}) - S(\Delta^{\text{L2}})
\ge k(\Delta^{\text{L2}} - \Delta^{\text{L3}})
= k \sum_{(i,t)\in\mathcal{U}} \delta_{i,t}.
\label{eq:stealth_dma_sma}
\end{equation}
On the other hand, the difference in attack cost is
\begin{align}
C(x^{\text{L3}}) - C(x^{\text{L2}})
&= \sum_{i,t} c_{i,t}
\left( x^{\text{L3}}_{i,t} - x^{\text{L2}}_{i,t} \right) \notag \\
&= - \sum_{(i,t)\in\mathcal{U}} c_{i,t}.
\label{eq:cost_dma_sma}
\end{align}

Therefore, the difference in AEF function satisfies
\begin{align}
J_{\text{L3}} - J_{\text{L2}}
&= \alpha \!\left[ D(x^{\text{L3}}) - D(x^{\text{L2}}) \right]
 + \beta \!\left[ S(\Delta^{\text{L3}}) - S(\Delta^{\text{L2}}) \right] \notag \\
&\quad - \gamma \!\left[ C(x^{\text{L3}}) - C(x^{\text{L2}}) \right].
\label{eq:aef_diff_def}
\end{align}

By substituting Eq.~\eqref{eq:disrupt_dma_sma}, Eq.~\eqref{eq:stealth_dma_sma}, and Eq.~\eqref{eq:cost_dma_sma} into Eq.~\eqref{eq:aef_diff_def}, we obtain
$J_{\text{L3}} - J_{\text{L2}} \ge -\sum_{(i,t)\in\mathcal{U}} \Phi_{i,t}$. By the assignment rule of SMA, any agent-time pair $(i,t)\in\mathcal{U}$ must satisfy $\Phi_{i,t} \le 0$, hence $-\Phi_{i,t} > 0$, and we further infer
$J_{\text{L3}} - J_{\text{L2}} \ge - \sum_{(i,t)\in\mathcal{U}} \Phi_{i,t} > 0$.
Hence, we have 
\begin{equation*}
J_{\text{L3}} > J_{\text{L2}} .
\end{equation*}
Furthermore, if there is no agent-time pair
$(i,t)\in\mathcal{U}$ with $\Phi_{i,t} < 0$, 
then $J_{\text{L3}} = J_{\text{L2}}$.

\end{proof}

\section{Episode Length Analysis}
\label{app:D}
We report the average episode length on SMAC II MMM map for the three collusive adversarial attacks.
As shown in Table~\ref{tab:episode_length}, a clear and consistent trend can be observed across all experimental settings: as the attack mode evolves from CMA to DMA and further to SMA, the average episode length gradually increases.

CMA tends to impose intensive and continuous adversarial perturbations on the normal cooperative structure within a short period, which more rapidly drives the system toward failure or termination. In contrast, DMA and SMA reduce the frequency of attack triggering or distribute the attack execution among different malicious agents. As a result, their influence on the overall trajectory becomes more sparse, allowing episodes to last longer.

Notably, the increase in episode length reflects a transition from short-term strong disruption to long-term stealthy interference.

\begin{table}[!htp]
\centering
\caption{Average episode length on the MMM map under different collusive adversarial attacks.}
\label{tab:episode_length}

\renewcommand{\arraystretch}{1.2}
\setlength{\tabcolsep}{8pt}
{
\begin{tabular}{c c c c}
\toprule
$\mathbf{n}$ & \textbf{CMA}& \textbf{DMA}& \textbf{SMA}\\
\midrule
2 & 55.45& 74.85& 126.95\\
3 & 51.50& 62.00& 97.25\\
4 & 50.15& 55.34& 95.90\\
\bottomrule
\end{tabular}
}
\end{table}

\section{Hyperparameters}
\label{app:F}

To ensure a fair comparison, all attacks are evaluated under the same training framework, network architecture, and core hyperparameter settings.  

As shown below, Table~\ref{tab:shared_hyper} summarizes the shared hyperparameters used by the three attacks in our CAMA framework, as well as all baselines in SMAC II.  
Table~\ref{tab:cama_hyper} lists the hyperparameters specific to the CAMA framework.  
Table~\ref{tab:baseline_hyper} lists the hyperparameters specific to the baselines.
\begin{table}[!htp]
\centering
\caption{Shared hyperparameters for CAMA and baselines.}
\label{tab:shared_hyper}

\renewcommand{\arraystretch}{1.2}
\setlength{\tabcolsep}{6pt}
\resizebox{\columnwidth}{!}{
\begin{tabular}{l c | l c}

\hline
\textbf{Hyperparameter} & \textbf{Value} & \textbf{Hyperparameter} & \textbf{Value} \\
\hline
Discount factor $\gamma$      & 0.99   & Actor network        & MLP \\
PPO epochs                    & 4      & Hidden dimension     & 64  \\
PPO clip ratio                & 0.2    & Activation function  & ReLU \\
Entropy coefficient           & 0.01   & Parallel environments & 16  \\
Learning rate                 & $1\times10^{-4}$ & Hidden layers        & 1   \\
Optimizer                     & Adam   & Gradient norm clip   & 10  \\
\hline
\end{tabular}
}
\end{table}

\begin{table}[!htp]
\centering
\caption{Specific hyperparameters for CAMA.}
\label{tab:cama_hyper}

\renewcommand{\arraystretch}{1.15}
\setlength{\tabcolsep}{6pt}
\resizebox{\columnwidth}{!}{
\begin{tabular}{l c | l c}
\hline
\multicolumn{4}{c}{\textbf{Hyperparameters for CMA}} \\
\hline
\textbf{Hyperparameter} & \textbf{Value} 
& \textbf{Hyperparameter} & \textbf{Value} \\
\hline
collusion\_nhead& [2, 3, 4]
& collusion\_alpha& [0.1, 0.2, 0.3] \\
\hline

\multicolumn{4}{c}{\textbf{Hyperparameters for DMA}} \\
\hline
\textbf{Hyperparameter} & \textbf{Value} 
& \textbf{Hyperparameter} & \textbf{Value} \\
\hline
$\tau_{\text{avg}}$ & [$-0.30$, $-0.40$]
& & \\
\hline

\multicolumn{4}{c}{\textbf{Hyperparameters for SMA}} \\
\hline
\textbf{Hyperparameter} & \textbf{Value} 
& \textbf{Hyperparameter} & \textbf{Value} \\
\hline
$\tau_{\text{ind}}$ & [$-0.30$, $-0.40$]
& & \\
\hline
\end{tabular}
}
\end{table}


\begin{table}[!htp]
\centering
\caption{Specific hyperparameters for baselines.}
\label{tab:baseline_hyper}

\renewcommand{\arraystretch}{1.15}
\setlength{\tabcolsep}{15pt}
\resizebox{\columnwidth}{!}{
\begin{tabular}{l c  |l c}
\hline
\multicolumn{4}{c}{\textbf{Hyperparameters for AMI.}} \\
\hline
\textbf{Hyperparameter} & \textbf{Value}
& \textbf{Hyperparameter} & \textbf{Value} \\
\hline
rand\_sensitivity& 1
& target\_sensitivity& 0.1 \\
target\_sensitivity& 10
& social\_influence& 10000 \\
\hline

\multicolumn{4}{c}{\textbf{Hyperparameters for IMAP}} \\
\hline
\textbf{Hyperparameter} & \textbf{Value}
& \textbf{Hyperparameter} & \textbf{Value} \\
\hline
imap\_intrinsic\_coef & 0.3
& imap\_k             & 10 \\
imap\_cache\_size     & 100
& imap\_chunk\_size   & 256 \\
\hline
\end{tabular}
}
\end{table}
\section{Additional Tracker Evaluation Metrics}
\label{app:G}
Following the exposure-based analysis in the main text, we further
report two complementary detection metrics: First Detection Time
and Detection Rate, as illustrated in Fig.~\ref{fig:ablation_time_rate}.
\begin{figure}[t]
    \centering
    \begin{subfigure}[t]{1\columnwidth}
        \centering
        \includegraphics[width=\linewidth]{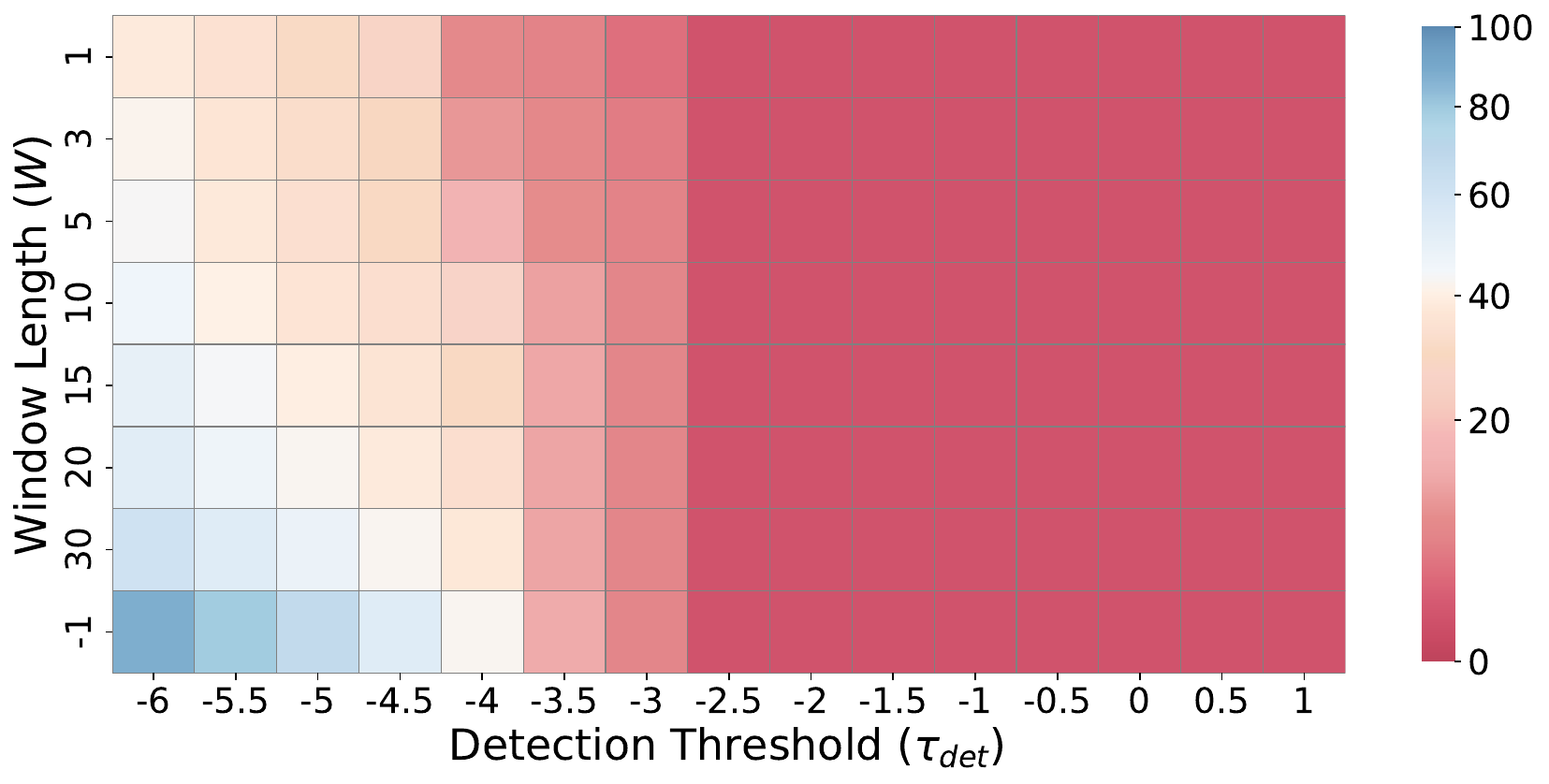}
        \caption{First Detection Time}
        \label{fig:ablation_dma_map1}
    \end{subfigure}
    \hfill
    \begin{subfigure}[t]{1\columnwidth}
        \centering
        \includegraphics[width=\linewidth]{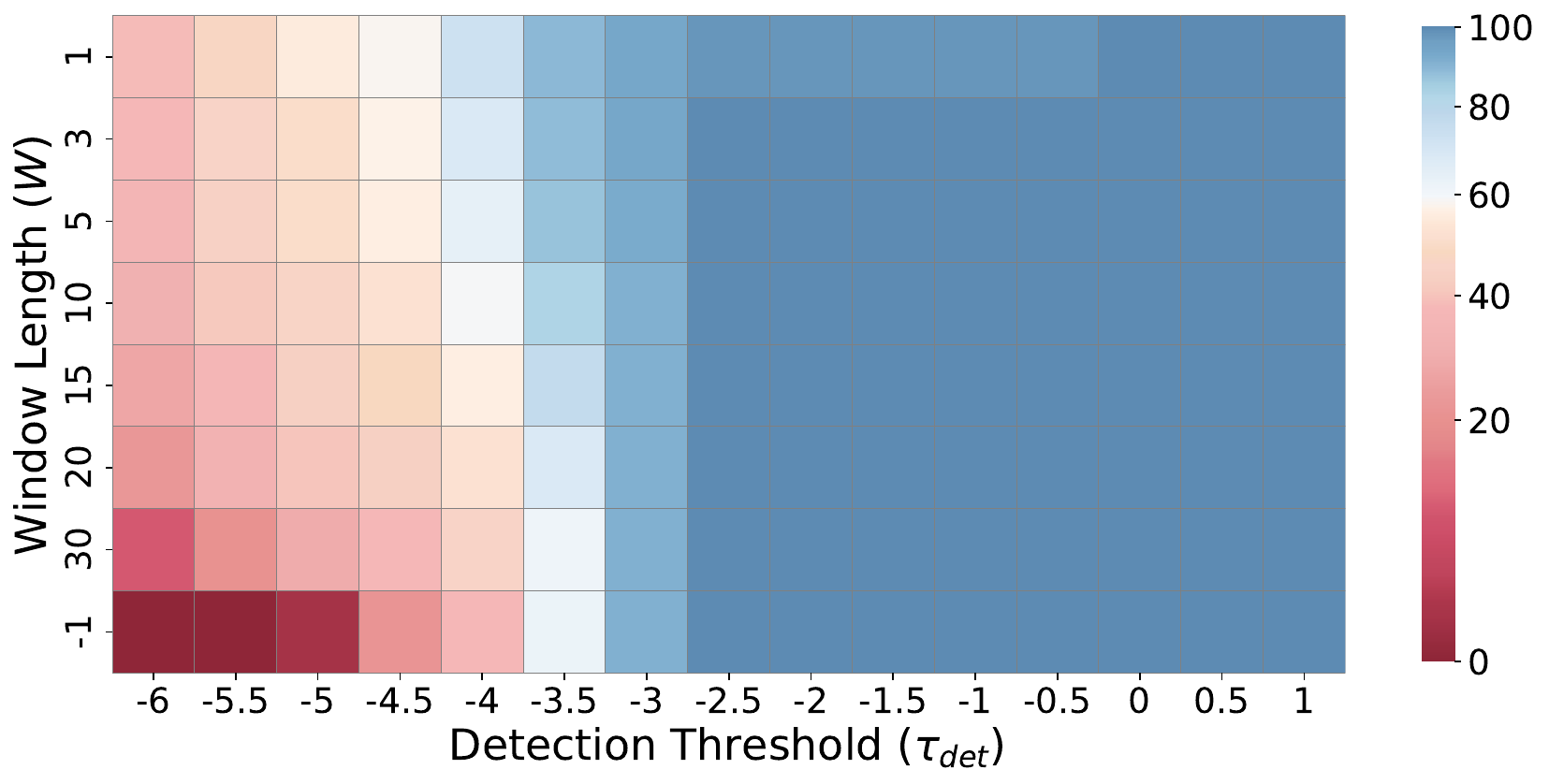}
        \caption{Detection Rate}
        \label{fig:ablation_dma_map2}
    \end{subfigure}

    \caption{First Detection Time and Detection Rate across different window lengths and detection thresholds.}
    \label{fig:ablation_time_rate}
\end{figure}
These metrics characterize the latency and frequency of successful
detections, respectively, under the same experimental settings.
In particular, unlike Exposure Intensity, which accumulates individual
exposure events over time, the Detection Rate measures the fraction of
time steps in which \emph{at least one} malicious agent is detected,
thereby reflecting detection effectiveness at the collective level of
the malicious group.


\end{document}